\documentclass[11pt]{amsart}

\usepackage[backref=page]{hyperref}
\usepackage[alphabetic]{amsrefs}

\usepackage[sc, osf]{mathpazo}
\usepackage[T1]{fontenc}
\usepackage{hyperref}
\hypersetup{colorlinks = true, linkcolor = {blue}, urlcolor  = blue}
\usepackage[all,2cell]{xy} \UseAllTwocells \SilentMatrices \SelectTips{cm}{}
\usepackage{latexsym,amsfonts,amssymb}
\usepackage{amsmath,amsthm,amscd}
\usepackage[dvips]{epsfig}
\usepackage{psfrag}

\usepackage{xcolor}
\usepackage{color}
\usepackage{etoolbox}
\usepackage{graphicx}
\usepackage{a4wide}
\usepackage[margin=1in]{geometry}
\usepackage{graphicx}
\usepackage{fancyhdr}
\usepackage{mathtools}
\usepackage{topcapt}
\usepackage{booktabs}
\usepackage{comment}
\usepackage[scaled]{helvet}
\usepackage[colorinlistoftodos]{todonotes}
\usepackage[T1]{fontenc}
\usepackage{fancyvrb}
\usepackage{listings}
\usepackage{color}
\usepackage{thmtools,thm-restate}

\usepackage{tikz}
\usetikzlibrary{arrows,chains,matrix,positioning,scopes}
\usetikzlibrary{matrix}
\usetikzlibrary{arrows.meta}
\tikzset{|/.tip={Bar[width=.8ex,round]}}

\DeclareMathAlphabet{\mathsfit}{\encodingdefault}{\sfdefault}{m}{sl}
\SetMathAlphabet{\mathsfit}{bold}{\encodingdefault}{\sfdefault}{bx}{n}

\usepackage[ruled]{algorithm2e}

\newcommand{\remspace}[1]{}

\usepackage{amsmath,amsfonts,bm, amsthm}
\usepackage{mathrsfs} 

\usepackage{tikz-cd}

\newcommand{\bfA}{\mathbf{A}}
\newcommand{\bfB}{\mathbf{B}}
\newcommand{\bfQ}{\mathbf{Q}}
\newcommand{\bfR}{\mathbf{R}}
\newcommand{\bfT}{\mathbf{T}}
\newcommand{\bfU}{\mathbf{U}}
\newcommand{\bfV}{\mathbf{V}}
\newcommand{\bfW}{\mathbf{W}}
\newcommand{\bfX}{\mathbf{X}}

\newcommand{\boldrho}{\boldsymbol \rho}

\newcommand{\bfa}{\mathbf{a}}

\newcommand{\bfk}{\mathbf{k}}

\newcommand{\bfn}{\mathbf{n}}
\newcommand{\bfd}{\mathbf{d}}

\newcommand{\bfdred}{\mathbf{d}^{\text{\rm red}}}

\newcommand{\dred}{d^{\rm red}}

\newcommand{\inv}{^{-1}}

\newcommand{\id}{\text{\rm Id}}
\newcommand{\red}{^{\text{\rm red}}}

\newcommand{\inx}{i}

\newcommand{\inc}{\text{\rm Inc}}

\newcommand{\Q}{\mathbb{Q}}
\newcommand{\R}{\mathbb{R}}

\renewcommand{\L}{\mathcal{L}}

\newcommand{\Rep}{\mathsf{Rep}}

\newcommand{\quiver}{\mathscr{Q}}

\newcommand{\bfdhid}{\bfd^\text{\rm hid}}

\newcommand{\biasvercol}{cyan}

\newcommand{\Par}{\mathsf{Param}}

\newcommand{\Parqd}{\Par(\quiver, \bfd)}

\newcommand{\Parqdred}{\Par(\quiver, \bfdred)}

\newcommand{\Parintqd}{\Par^{\text{\rm int}} (\quiver,  \bfd)}

\newcommand{\Proj}{\mathrm{Proj}}

\theoremstyle{plain}
\newtheorem{theorem}{Theorem}[section]
\newtheorem{prop}[theorem]{Proposition}

\newtheorem{lemma}[theorem]{Lemma}
\newtheorem{cor}[theorem]{Corollary}

\theoremstyle{definition}

\newtheorem{rmk}[theorem]{Remark}

\numberwithin{equation}{section}

\def\eqref#1{equation~\ref{#1}}

\def\1{\bm{1}}

\newcommand{\rwa}[1] {}

\DeclareMathAlphabet{\mathsfit}{\encodingdefault}{\sfdefault}{m}{sl}
\SetMathAlphabet{\mathsfit}{bold}{\encodingdefault}{\sfdefault}{bx}{n}

\usepackage[foot]{amsaddr}

\title[Quiver neural networks]{Quiver neural networks}
\author[Iordan Ganev and Robin Walters]{Iordan Ganev \qquad Robin Walters}
\address[IG]{\rm Radboud University, Nijmegen, Netherlands. \texttt{iganev@cs.ru.nl}}
\address[RW]{\rm Northeastern University, Boston, MA, USA. \texttt{r.walters@northeastern.edu}}

\begin{document}

\maketitle

\begin{abstract}
We develop a uniform theoretical approach towards the analysis of various neural network connectivity architectures by introducing the notion of a {\it quiver neural network}. Inspired by quiver representation theory in mathematics, this approach gives a compact way to capture elaborate data flows in complex network architectures. As an application, we use parameter space symmetries to  prove a lossless model compression algorithm for quiver neural networks with certain non-pointwise activations known as {\it rescaling activations}. 
In the case of {\it radial} rescaling activations, we prove that training the compressed model with gradient descent is equivalent to training the original model with projected gradient descent. 
\end{abstract}

\parskip = 10pt
\parindent = 0pt
	
\section{Introduction}\label{sec:intro}

In recent years, the study of deep neural networks has advanced from the basic case of sequential multilayer perceptrons to incorporating more elaborate structures, including  skip connections, multiple inputs with different features pathways, multiple output heads, aggregations, concatenations, and  splitting of features. As a result, terms such as `layer' and `depth' come across as insufficiently descriptive for capturing the interdependencies of the hidden feature spaces.  In this paper, we propose {\it quiver neural networks} as a convenient formal tool for schematically describing the data flow in complex network structures. 

Our approach stems from representation theory, a branch of mathematics concerned with the formal study of equivariance and symmetry. Representation-theoretic perspectives are increasingly influential  in deep learning theory, and have been especially successful in the context of equivariant neural networks, where one exploits symmetries of the input or output spaces, and considers distributions and functions that respect these symmetries  \cite{cohen2016group, kondor_generalization_2018, ravanbakhsh2017equivariance, cohen2016steerable}. Our techniques, by contrast, focus on parameter space symmetries, and consequently pertain to some degree to all neural networks. Our methods adopt constructions from quiver representation theory,  an active research area in mathematics with connections to Lie theory and symplectic geometry \cite{nakajima1998quiver, kirillov2016quiver}. We also build on earlier work applying quiver representation theory to deep learning \cite{armenta_double_2021, armenta_representation_2020, armenta_neural_2021, jeffreys_kahler_2021}.  

Formally, a {\it quiver} is another term for a finite directed graph. In our constructions, the vertices of the quiver indicate the layers of a neural network architecture, while the edges indicate the data flows interconnecting the layers.  A {\it representation} of a quiver is the assignment of a vector space (i.e., a hidden feature space) to each vertex and a compatible linear map (i.e., a weight matrix) to each edge. Quiver representations, together with a non-linearity at each hidden vertex, combine to produce the structure of a neural network, which we refer to a quiver neural network. Figure \ref{fig:MLP-NN} illustrates the procedure in the case of a sequential network with three layers. 

The formalism of quiver neural networks has several distinct advantages. First, it captures the underlying network connectivity in the simplest possible way, precipitating a uniform analysis of many different data flow architectures including multilayer perceptrons (MLPs) and other basic sequential networks; skip connections; multiple input, output, and internal pathways; aggregations; and others \cite{mohammed_y-net_2018, shelhamer_fully_2017, huang_densely_2017, zilly_recurrent_2017}. 
Second, the quiver perspective makes parameter space symmetries more apparent. Specifically, the space of quiver representations has a natural group action; in our formalism, this means one can package the trainable network parameters into a vector space with rich symmetries. 
Finally, factoring out parameter space symmetries leads to a geometric rephrasing of neural network optimization problems in terms of quiver varieties, where favorable convexity properties are conjectured to exist. 

\begin{figure}[t]
	\label{fig:MLP-NN}
	\begin{center}
	\centering
	\includegraphics[width=0.85\textwidth]{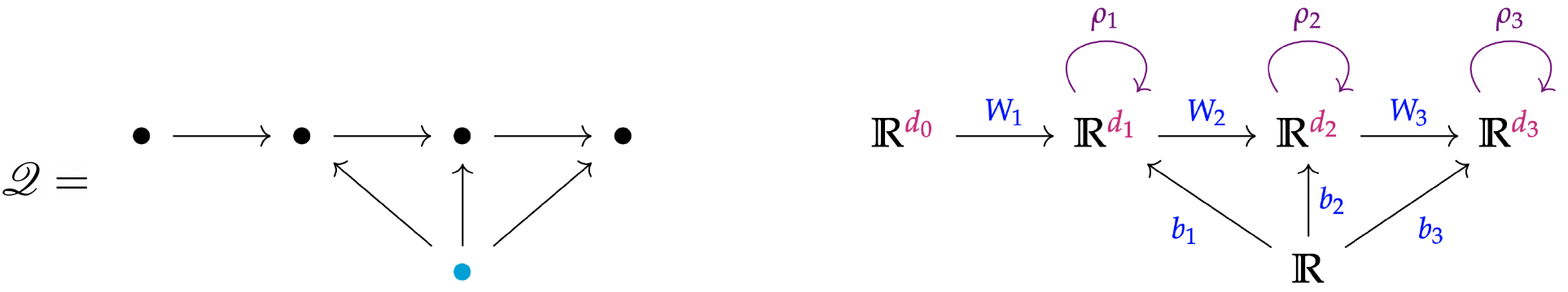}
\end{center}         
	
\caption{(Left) An example of a quiver.  Each vertex in the top row represents a layer; the bottom vertex is the `bias' vertex. (Right) A quiver neural network. We assign a dimension to each non-bias vertex, linear map to each edge, and an activation to each non-source vertex. When each $\rho_i$ is a pointwise activation, we obtain a three-layer multilayer perceptron. }
\end{figure}

As our main application, we generalize the model compression results of \cite{ganev_universal_2022} to quiver neural networks. These results apply to neural networks with {\it rescaling activations}, rather than usual pointwise activations. Such activations rescale each activation input by a scalar-valued function to produce the activation output.  A special class of rescaling activations are {\it radial} rescaling activations, where the rescaling factor depends only on the norm of the input vector. Neural networks with the radial rescaling version of ReLU (known as Step-ReLU, see Section \ref{sec:rescaling-MC}) are universal approximators \cite{ganev_universal_2022}.

Our model compression algorithm for quiver neural networks with rescaling activations proceeds via iterated QR decompositions over a topological ordering of the vertices. At each layer, one extracts the weights of the compressed model from the upper-triangular matrix `$R$', and uses the orthogonal matrix `$Q$' to modify the activation to a different rescaling function. The output of the algorithm is a quiver neural network with the same underlying quiver, a lower-dimensional feature space at each layer, and new rescaling activations. A key feature of our model compression algorithm is that it is {\it lossless}: the compressed network has the same value of the loss function on any batch of training data as the original network. 

Moreover, we obtain stronger results in the case of radial rescaling activations, which commute orthogonal transformations and hence enjoy special  equivariance properties. When the original model has radial activation functions, the compressed model has the same activations restricted to the appropriate subspace; that is, the compression algorithm does not require updating radial rescaling activations. Furthermore, use of radial activations leads to a precise interaction between model compression and training. While training the compressed model via gradient descent does not yield the same result as training the original model, we prove an explicit mathematical relationship between the two procedures, provided that (1) the model has radial rescaling activations and (2) one trains the original model with {\it projected gradient descent}.  As explained below, projected gradient descent involves zeroing out specific matrix entries after each step of gradient descent. When the compression is significant enough, the compressed model reaches a local minimum faster than the original model.

To summarize, our specific  contributions include:\vspace{-5pt}
\begin{enumerate}
	\item A theoretical framework for neural networks based on quiver representation theory;
	\item An implementation of a lossless model compression algorithm for quiver neural networks with rescaling activations. 
    \item A refinement of this algorithm for  radial rescaling activations, based on equivariance properties.
	\item A theorem relating gradient descent optimization of the original and compressed networks, in the case of radial rescaling activations. 
\end{enumerate}

\paragraph{\bf Terminology} To avoid ambiguity, we remark on terminology used in this paper. The terms `quiver' and `finite directed graph' are synonymous, but the terms `quiver neural network' and `graph neural network' are distinct. In the latter, the feature spaces are function spaces on the vertex set of a graph \cite{kipf2016semi}. By contrast, for quiver neural networks, the feature spaces are not assumed to have any special structure, while the network connectivity  (number of layers, interconnections between the layers) is specified by a quiver. Additionally, we do not use the expression `representation of a quiver' in technical sections of the main text, only as background and in the appendix. Instead, we consider the parameter space associated to a quiver.

\section{Related work}\label{sec:related-work}	

\paragraph{\bf Quiver representation theory and neural networks.}
This work generalizes results in \cite{ganev_universal_2022} from the case of basic sequential networks to any quiver neural network (an earlier version of loc.\ cit.\ discussed quivers). 
There has been a number of ground-breaking works relating neural networks to quiver representations \cite{armenta_representation_2020, armenta_double_2021, armenta_neural_2021}. As far as we can tell, our current paper strengthens these relations as it (1) captures all architectures appearing thus far, (2) accommodates both pointwise and non-pointwise activation functions, and (3) proves generalizations to larger symmetry groups.
Our work is also influenced by algebro-geometric and categorical perspectives placing quiver varieties in the context of neural networks  \cite{jeffreys_kahler_2021, manin_homotopy_2020}. At the same time, we place more emphasis  practical consequences for optimization techniques at the core of machine learning. 
Our approach also shares parallels with works where quivers are not explicitly mentioned. 
For example, \cite{wood_representation_1996} consider neural networks in which each layer is a representation of a finite group, and activations are pointwise;  our  framework, however, accommodates Lie groups as well as non-pointwise activations.  Also, special cases of the ``algebraic neural networks'' of \cite{parada-mayorga_algebraic_2020} are equivalent to representations of quivers over rings of polynomials. Finally, the study of the ``non-negative homogeneity'' (or ``positive scaling invariance'') property of ReLU activations \cite{dinh_sharp_2017, neyshabur_path-sgd_2015, meng_g-sgd_2019}  is a special case of the results appearing in our work. 

\paragraph{\bf Rescaling activations.}  The special case of radial rescaling activations have been studied from several perspectives. Radial rescaling   functions have the symmetry property of preserving vector directions, and hence exhibit rotation equivariance.  Consequently, examples of such functions, such as the squashing nonlinearity and Norm-ReLU,  feature in the study of  rotationally equivariant neural networks   \cite{weiler_general_2019, sabour2017dynamic, weiler20183d, weiler2018learning, jeffreys_kahler_2021}.  
From a different direction, in the  vector neurons formalism \cite{deng_vector_2021} the output of a nonlinearity is a vector rather than a scalar; rescaling activations are an example. 
For radial basis networks, each hidden neuron is a radial nonlinear function of the shifted input vector, but the outputs are independent, whereas for rescaling functions, the outputs are also linked together \cite{broomhead1988radial}. 

\paragraph{\bf Equivariant neural networks.} Representation-theoretic techniques appear in the development of equivariant neural networks; such networks are designed to incorporate symmetry as an inductive bias.  In particular, equivariant networks feature weight-sharing constraints based on equivariance or invariance with respect to various symmetry groups. Examples of equivariant architectures include   $G$-convolution, steerable CNN, and Clebsch-Gordon networks \cite{cohen2019gauge,  weiler_general_2019, cohen2016group, chidester2018rotation, kondor_generalization_2018, bao2019equivariant, worrall2017harmonic, cohen2016steerable, weiler2018learning, dieleman2016cyclic, lang2020wigner, ravanbakhsh2017equivariance}.  By contrast, the quiver representation theory approach taken in this paper does not depend on symmetries occurring in the input domain, output space, or feedforward mapping.  Instead, we exploit parameter space symmetries and thus obtain more general results that apply to domains with no apparent symmetry.  From the point of view of  model compression, equivariant networks do achieve reduction in the number of trainable parameters through weight-sharing for fixed hidden dimension widths; however, in practice, they may use larger layer widths and consequently have larger memory requirements than non-equivariant models. Sampling or summing over large symmetry groups may make equivariant models computationally slow as well \cite{finzi2020generalizing, kondor_generalization_2018}.      

\paragraph{\bf Model compression.}
Apart from \cite{ganev_universal_2022} mentioned above, our method differs significantly from most existing model compression methods  \cite{cheng2017survey} in that it is based on the inherent symmetry of neural network parameter spaces. 
One prior model compression method is \emph{weight pruning}, which removes redundant, small, or unnecessary weights from a network with little loss in accuracy \cite{han2015deep, blalock2020state, karnin1990simple}.  Pruning can be done during training or at initialization \cite{frankle2018lottery,	lee2019signal, wang2020picking}.  \emph{Gradient-based pruning} identifies low saliency weights by estimating the increase in loss resulting from their removal \cite{lecun1990optimal, hassibi1993second, dong2017learning, molchanov2016pruning}.    A complementary approach is \emph{quantization}, in which aims to decrease the bit depth of weights \cite{wu2016quantized, howard2017mobilenets, gong2014compressing}.    \textit{Knowledge distillation} works by training a small model to mimic the performance of a larger model or ensemble of models \cite{bucilua2006model,hinton2015distilling, ba2013deep}.
\textit{Matrix Factorization} methods replace fully connected layers with lower rank or sparse factored tensors \cite{cheng2015fast, cheng2015exploration, tai2015convolutional, lebedev2014speeding, rigamonti2013learning, lu2017fully} and can often be applied before training.   Our method  involves a generalized QR decomposition, which is a type of matrix factorization; however, rather than aim for a rank reduction of linear layers, we leverage this decomposition in order to reduce hidden layer widths. Our method shares similarities with lossless compression methods which aim to remove stable or redundant neurons \cite{serra2021scaling, serra2020lossless}, sometimes using symmetry \cite{sourek2020lossless}. Finally, while there are similarities between our  model compression  results and those of \cite{jeffreys_kahler_2021}, our results apply to a larger class of activations (beyond the squashing nonlinearity) and feature a group action on all layers (not just disgjoint layers). 

\section{Quiver neural networks}\label{sec:quiver-NNs}

\subsection{Motivation}

What is a neural network? From an abstract point of view, it consists of (1) a connectivity architecture indicating the arrangement and interdependencies of the layers, as well as the data flow, (2) trainable parameters capturing the linear component of the feedforward function, and (3) non-linearities that enhance the expressivity of the feedforward function. As we observe below, the connectivity architecture of a neural network corresponds to a  quiver, that is, a finite directed graph. Meanwhile, the trainable parameters define a representation of the quiver, that is, a vector space for every vertex and a compatible matrix for every edge. Finally, the non-linearities appear as transformations of the vector space at each vertex. As we will see, this perspective has the advantages of compactly defining the underlying neural network structure, packaging the trainable parameters of a neural network into a vector space, and emphasizing the parameter space symmetries, which are crucial for our results.  

\begin{figure}[t]
	\begin{align*}
	\resizebox{0.85\textwidth}{!}{	\xymatrix{ \bullet \ar[r] &\bullet \ar[r]  & \bullet \ar[r] & \bullet \ar[r] & \bullet \ar[r] &  \bullet \\ 
			&  &   & {\color{\biasvercol} \bullet} \ar[rru] \ar[ru] \ar[u] \ar[lu] \ar[llu]  &  &    }
		\qquad \qquad 
		\xymatrix{ \bullet \ar[r] &\bullet \ar[r] \ar@/^1pc/[rr]  & \bullet \ar[r] & \bullet \ar[r] \ar@/^1pc/[rr]    & \bullet  \ar[r] & \bullet  \\ 
			&  &   & {\color{\biasvercol} \bullet}  \ar[rru] \ar[ru] \ar[u] \ar[lu] \ar[llu]  &  &    } }\\
	\resizebox{0.85\textwidth}{!}{\xymatrix{ & \\ \bullet \ar[r] &\bullet \ar[r] \ar@/^2pc/[rrrr]  & \bullet \ar[r] \ar@/^1pc/[rr]  & \bullet \ar[r]    & \bullet  \ar[r] & \bullet    		\\ 
			&  &   & {\color{\biasvercol} \bullet}  \ar[rru] \ar[ru] \ar[u] \ar[lu] \ar[llu]  &  &    }
		\qquad \qquad 
		\xymatrix{ \bullet \ar[r]  & \bullet \ar[r] & \bullet  \ar[rd] \\
			& 	 & &\bullet \ar[r]  & \bullet \ar[r] & \bullet \\
			\bullet \ar[r]  & \bullet \ar[r] & \bullet\ar[ru] \\
			& & & {\color{\biasvercol} \bullet}  \ar[rruu] \ar[ruu] \ar[uu] \ar[lu] \ar[llu] \ar[luuu] \ar[lluuu]  &  &  
	} }
	\end{align*}
	\caption{Examples of neural quivers.  The bias is indicated in blue. Clockwise from top left: the MLP quiver, a quiver with skip connections, a U-net quiver, a quiver defining a network with a  "Y" structure.}
 	\label{fig:quivers}

\end{figure}
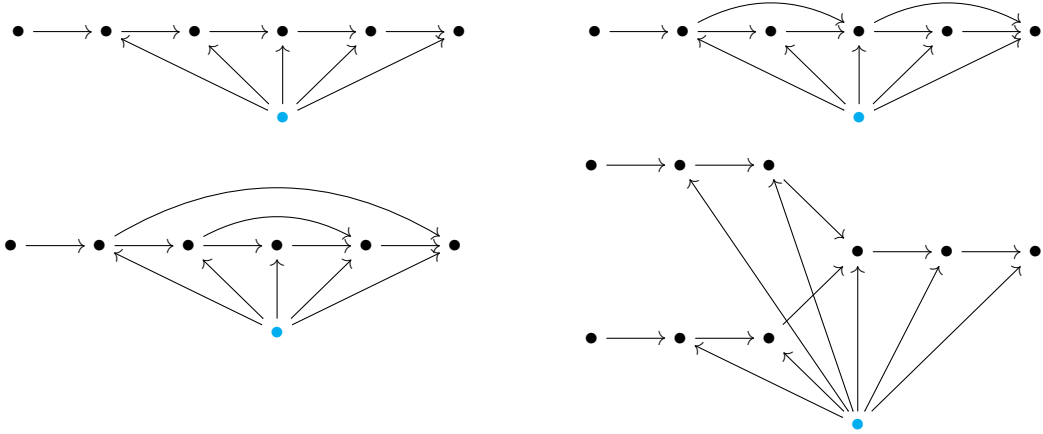

\subsection{Quivers}
A {\it quiver} $\quiver$ is a finite directed graph.
Thus, a quiver consists of a pair $\quiver = (I, E)$, where $I$ is the set of vertices and  $E$ is the set of directed edges, together with  maps $ s: E \rightarrow I$ and $t: E \to I$ indicating the source and target of each edge. 
As usual, a {\it source} in $\quiver$ is a vertex with no incoming edges, and a {\it sink} is a vertex with no outgoing edges. 
A {\it dimension vector} for a quiver  is the assignment of a dimension $d_i$ to each vertex $i$; we group these into  a tuple $\bfd = (d_i)_{i \in I}$ of positive integers indexed by the vertices. 

\subsection{Neural quivers}\label{subsec:neural-quivers}
We  define the general class of quivers relevant to neural networks and machine learning. 
A {\it neural quiver} is a connected acyclic quiver $\quiver = (I, E)$, together with a distinguished vertex $i_\text{\rm bias} \in  I$, called the {\it bias vertex}, satisfying the following condition: The bias vertex $i_\text{\rm bias}$ is a source, and  removing $i_\text{\rm bias}$ creates no new sources in $\quiver$. The condition on the bias vertex guarantees that every non-source vertex admits a path from a non-bias source vertex. We note that all results have versions in the simpler case of no bias, and   any connected acyclic quiver with more than one vertex can be extended to the structure of a neural quiver by adding an extra source vertex to play the role of the bias. 
A vertex of a neural quiver is called {\it hidden} if it is neither a source nor a sink. Let $I_\text{\rm hidden}$ denote the set of hidden vertices.
Since any neural quiver is acyclic, it admits a {\it topological order} of its vertices, i.e., an enumeration $I = \{ i_1, \dots, i_{|I|}\}$ of the vertex set such that, if there is an edge from $i_n$ to $i_m$, then $n < m$. In many of the constructions that appear below, one must fix a topological order; however, the particular choice of topological order is irrelevant.

\subsection{Quiver neural networks}\label{subsec:quiver-NNs}

We now state the definition of a quiver neural network.
Let $\quiver = (I,E)$ be a neural quiver.  A {\it $\quiver$-neural network} consists of a dimension $d_i$ for each vertex $i \in I$, with $d_{i_\text{\rm bias}} = 1$; a weight matrix $W_e \in \R^{d_{t(e)} \times d_{s(e)}}$ for every edge $e \in E$; and an activation $\rho_i : \R^{d_i} \to \R^{d_i}$ for every non-source vertex $i$. 
We group the dimensions, weight matrices, and activations into tuples: $ \bfd = (d_i)_{i \in I}$, $\bfW = (W_e)_{e \in E}$, and $\boldrho = (\rho_i)_{i \in I}$ 
where, for convenience, we set $\rho_i$ to be the identity if $i$ is a source vertex. Hence, we denote a $\quiver$-neural network as a triple of tuples: $(\bfd, \bfW, \boldsymbol{\rho})$. 
The tuple $\bfd$ is a {dimension vector} for $\quiver$. Note that the `weight matrix' assigned to an edge from the bias vertex to a vertex $i$ is of size $d_i \times 1$, so it is a column vector, or, equivalently, an element of $\R^{d_i}$. Examples of quiver neural networks include:

\begin{enumerate}
	\setlength\itemsep{5pt}
	\item  Sequential quivers, see Figure \ref{fig:MLP-NN} and Figure \ref{fig:quivers}, top left. 	In Figure \ref{fig:MLP-NN}, the weight matrix $W_i$ is of size $d_{i} \times d_{i-1}$, and the `matrix' $b_i$ corresponds to an element of $\R^{d_i}$. 
Each activation $\rho_i$ is a non-linear map from $\R^{d_i}$ to itself. When the activation functions are pointwise, the resulting quiver neural network is nothing more than an MLP. 

	\item ResNet quivers, see Figure \ref{fig:quivers}, top right and bottom left. These quivers have skip-connections, and  the resulting quiver neural networks are examples of ResNets.

	\item Multiple inputs are possible, as in Figure \ref{fig:quivers}, bottom right. For example, Y-nets are captured by our formalism \cite{mohammed_y-net_2018}.

	\item U-nets have been studied in \cite{shelhamer_fully_2017}, and are also examples of quiver neural networks, as indicated in Figure \ref{fig:quivers} bottom left\footnote{At least the non-convolutional versions thereof. We expect the quiver formalism to generalize to convolutional neural networks.}. 

\end{enumerate}

\subsection{The feedforward function}\label{subsec:feedforward}

To define the feedforward function of a quiver neural network $(\bfd, \bfW, \boldrho)$, we require some additional terminology.
We call any source vertex that is not the bias an {\it input} vertex, and any sink vertex an {\it output} vertex. Thus, the vertex set $I$ is the disjoint union of  $\{i_\text{\rm bias} \}$, $\{\text{\rm input vertices}\}$, $\{\text{\rm output vertices}\}$, and  $I_\text{\rm hidden}$. 
 
Given a dimension vector $\bfd$ for $\quiver$, set:
$d_\text{\rm in} = \sum_{i \in \{\text{\rm inputs}\}} d_i$ and $ d_\text{\rm out} = \sum_{i \in \{\text{\rm outputs}\}} d_i$
Hence, we identify $\R^{d_\text{\rm in}}$ with the direct sum of the spaces $\R^{d_i}$ as $i$ ranges over the input vertices, and similarly for $\R^{d_\text{\rm out}}$. 
Next, fix a topological order of the vertices. For every vertex $i \in I$, define a function:
$ F_i : R^{d_\text{\rm in}} \to \R^{d_i}$ 
recursively as follows. If $i$ is a source, then $F_{i}$ is the projection map. If $i = i_\text{\rm bias}$ is the bias vertex, then $F_{i} \equiv 1$  is the constant function at $1 \in  \R = \R^{d_{i_\text{\rm bias}}}$.  For all other vertices, we use the recursive definition:
	$$ F_{i}(x) = 	\rho_i  \left( \sum_{j \stackrel{e}{\rightarrow}  i} W_e \circ  F_j(x)   \right).$$
That is, to compute the hidden feature at node $i$, first loop over the incoming edges to $i$, applying the weight matrix for each edge to the feature vector at the source vertex of the edge; next, sum the results (which are all elements of $\R^{d_i}$); and, finally, apply the activation. 
The definitions of  the $F_i$ are independent of the choice of topological order, as is the following definition:	The {\it feedforward function} $F =F_{(\bfd, \bfW, \boldrho)}  : \R^{d_\text{\rm in}} \rightarrow \R^{d_\text{\rm out}}$ of the neural network $(\bfd, \bfW, \boldrho)$  is defined as $F(x) =  ( F_{i}  (x) )_{i \in  \{\text{\rm output vertices}\}}.$

\section{Rescaling activations and QR model compression}\label{sec:rescaling-MC}

In this section, we state a model compression result for quiver neural networks in which each activation belongs to a certain class of non-pointwise activations, known as a rescaling activations. We also consider the case of radial rescaling activations, which commute with orthogonal transformations. Proofs appear in Appendix \ref{app:model-compression}. 

\subsection{Rescaling activations}

A function $\rho : \R^d \to \R^d$ is said to be {\it rescaling} if it sends each vector to a scalar multiple itself, that is: $$\rho(v) = \lambda(v)v \qquad \text{\rm for all $v \in \R^d$}$$ for some scalar-valued function $\lambda : \R^d \to \R$.  We say that a $\quiver$-neural network $(\bfd, \bfW, \boldrho)$ has {\it rescaling activations} if each $\rho_i : \R^{d_i} \to \R^{d_i}$ is a rescaling function. Examples include: (1) Step-ReLU, where the rescaling factor $\lambda(v)$ is equal to $0$ if the vector $v$ is of norm less than one, and $1$ if $v$ of norm at least one.  Fully-connected MLPs with Step-ReLU activations are known to be universal approximators \cite{ganev_universal_2022}.  (2) $\lambda(v) = |v- c|$ for some $c \in \R^d$. 

\begin{figure}
	\centering
	\includegraphics[width=0.8\textwidth]{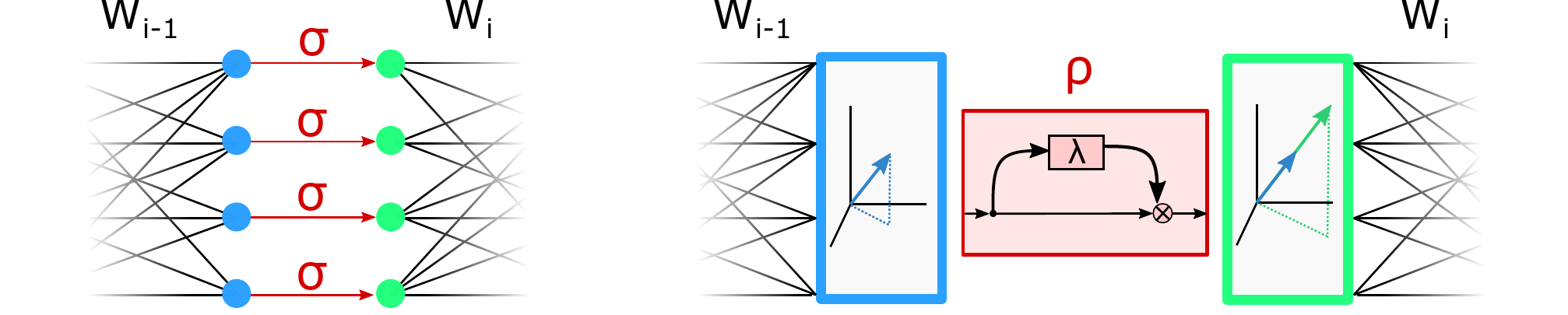}
	\caption{Figure 1: (Left) Pointwise activations distinguish a specific basis of each hidden layer and treat each coordinate independently. (Right) Rescaling activations	rescale each feature vector by a scalar-valued function $\lambda$. }
	\label{fig:rescaling}
\end{figure}

\subsection{Reduced dimension vector}
Let $\bfd$ be a dimension vector for an acyclic quiver $\quiver$. The {\it reduced} dimension vector $\bfdred$ associated to $\bfd$ is defined as follows, using a topological order of the vertices. If $i$ is a source or sink vertex, then $\dred_i = d_i$. Otherwise, $\dred_i$ is the minimum of  $d_i$ and the sum  $\dred_{\to i} = \sum_{j \to i} \dred_j$ of the values of the reduced dimension vector at all the vertices with an outgoing edge to $i$:
$$\dred_i = \min\left(d_i, \dred_{\to i} \right).$$ 
This definition is independent of the choice of topological order.  Since $\dred_i \leq d_i$, let  $\inc_i : \R^{\dred_i} \hookrightarrow \R^{d_i}$ be the inclusion into the first $\dred_i$ coordinates; as a matrix, it has ones along the diagonal and zeros elsewhere. Let $\pi_i : \R^{d_i} \twoheadrightarrow \R^{d\red_i}$ be the projection onto the first $\dred_i$ coordinates. 

\subsection{Model compression}

In this section, we give  a model compression result for quiver neural networks with radial activations. The result is based on Algorithm \ref{alg:QR-rescaling}, whose input is a quiver neural network rescaling activations and whose output is a quiver neural network with  smaller widths (and the same underlying quiver). Specifically, if the original widths are given by the dimension vector $\bfd$, then the compressed widths are given by the reduced dimension vector $\bfdred$. The algorithm proceeds by computing successive QR decompositions according to a topological order of the vertices. The resulting upper-triangular matrix leads to the reduced weights, while the resulting orthogonal matrix provides a change-of-basis for the feature spaces and is used to update the rescaling activations to new rescaling activations. 

\begin{theorem}\label{thm:mod-comp-rescaling}
	Algorithm \ref{alg:QR-rescaling} is a lossless model compression algorithm. More precisely, the input and output $\quiver$-neural networks of Algorithm \ref{alg:QR-rescaling} have the same feedforward function. 
\end{theorem}

\begin{algorithm}[t]
	\SetKwFunction{QRdecompCom}{QR-decomp}
	\SetKwFunction{QRdecompRed}{QR-decomp}
	\SetKwInOut{Input}{input}
	\SetKwInOut{Output}{output}
	\SetKwInOut{Initialize}{initialize}
	\DontPrintSemicolon
	
	\Input {$\quiver$-neural network $(\bfd, \bfW, \boldrho)$ with each $\rho_i$ rescaling}
	\Output {$\quiver$-neural network $(\bfdred, \bfW\red, \boldsymbol{\tau})$,   orthogonal matrices $\bfQ = (Q_i \in O(d_i))_{i \in I_\text{\rm hidden}}$ }
	\BlankLine
	$\bfQ, \bfW\red, \boldsymbol{\tau} \gets [\ ], [\ ], [\ ]$  
	\tcp*[r]{initialize output matrix lists}

 $I = \mathtt{top\_sort}(I)$  \tcp*[r]{topological order of vertices}
 
	\For{\text{\rm $i$ in $I$}    }{   
        $M_i \gets \mathtt{hstack}\left( W_e   Q_{s(e)}  \inc_{s(e)} \ : \ t(e) = i \right)$ \tcp*[r]{merge transf.\ incoming weights}

		\uIf{\text{\rm $i$ is not a source or a sink}}{
			$Q_\inx, R_\inx \gets $ \QRdecompCom{$M_i$, \ $\mathtt{mode = `complete'}$} \tcp*[r]{$M_{i} = Q_\inx  \inc_\inx  R_\inx$}
			$\tau_i \gets \pi_i \circ Q_i\inv \circ \rho_i \circ Q_i \circ \inc_i$ \tcp*[r]{update  activations}

   			Append $Q_i$ to $\bfQ$\;

		    Append $\tau_i$ to $\boldsymbol{\tau}$
		
		}
		\uElseIf{\text{\rm $i$ is sink}}{
			$R_i \gets M_i$\;
		    Append $\rho_i$ to $\boldsymbol{\tau}$ \tcp*[r]{unchanged output activations}
		}

		\For{$e$ \text{\rm such that} $t(e) = i$}{
	       $W_e\red \gets \mathtt{extract}\left( R_i, e \right)$ \tcp*[r]{extract columns corresp.\ to $e$}		
   
			Append $W\red_e$  to $\bfW\red$ \tcp*[r]{reduced weights}		
					}
	}	
 
	\KwRet $(\bfdred, \bfW\red, \boldsymbol{\tau})$, and  $\mathbf{Q}$
	\caption{QR Model Compression for Rescaling activations (\texttt{QR-Compress})}
	\BlankLine
	\hrule
	\BlankLine
	
	Notes:  The method \texttt{QR-decomp} returns the complete QR decomposition of the $d_i \times \dred_{\to i}$ matrix $M_i$ as  $Q_\inx  \inc_i  R_\inx$ where $Q_i \in O(d_\inx)$ is an orthogonal $d_i \times d_i$ matrix and  $R_i$ is upper-triangular of size $\dred_ \inx \times \dred_{\to i}$. The method $\mathtt{top\_sort}$ produces a topological order, the method  $\mathtt{hstack}$ concatenates matrices horizontally, and the method $\mathtt{extract}$ applied to $(R_i, e)$ extracts from $R_i$ the $\dred_{s(e)}$ columns corresponding to the edge $e$.

	\label{alg:QR-rescaling}
\end{algorithm}

While a full proof appears in Appendix \ref{app:model-compression}, we now give  a brief description of the ideas behind the proof. Let $i$ be a hidden vertex. Consider the subspace $\mathcal{V}_i$ of the feature space $\R^{d_i}$ spanned by the images of the linear maps $W_e$ corresponding to the incoming edges. Rescaling activations preserve subspaces, so the subspace $\mathcal{V}_i$ is sent to itself under the activation $\rho_i$. Hence, if $\mathcal{V}_i$ is a proper subspace (i.e., not all of $\R^{d_i}$), then one can ignore elements in $\R^{d_i}$ not in $\mathcal{V}_i$ and reduce $d_i$ to the dimension of $\mathcal{V}_i$. One must subsequently rewrite the matrices $W_i$ in a basis for $\mathcal{V}_i$; the QR decomposition provides a convenient way to do this. The resulting orthogonal matrix $Q_i$ is not relevant for the statement of Theorem \ref{thm:mod-comp-rescaling}, but features in the proof and in Section \ref{sec:proj-gd}. 

We remark on refinements and improvements to Algorithm \ref{alg:QR-rescaling} and Theorem \ref{thm:mod-comp-rescaling}; see Appendix \ref{app:model-compression} for more details. First, if  each  merged matrix $M_i$ is of full rank, as is practically always the case with random initialization, then no further lossless compression beyond the dimension vector $\bfdred$ is possible. In fact, there is a precise sense in which the compressed model is the {\it minimal subnetwork} of the original model with the same feedforward function (Appendix \ref{appsubsec:subnetworks}). However, in situations where $M_i$ are not assumed to be of full rank, Algorithm \ref{alg:QR-rescaling} can be improved to compress to a model with even narrower widths (Appendix \ref{appsubsec:improvement}). Finally, if one allows for changing the basis of the output space, then one can set $\dred_i$ equal to $\min(d_i, \dred_{\to i})$ for output vertices. This is due to the fact that  the common feedforward function of the compressed and original models lies in a subspace of $\R^{d_\text{\rm out}}$ of dimension 
$\sum_{i \in \text{\rm sinks}} \min(d_i, \dred_{\to i})$. 

\subsection{Radial neural networks}\label{subsec:RadNNs}
We now consider  a special class of rescaling functions with favorable equivariance properties. A {\it radial rescaling function} is a rescaling function of the form 
$$ \rho : \R^{d} \to \R^{d}, \qquad \qquad v \mapsto \frac{h(|v|)}{|v|}v$$
for a function $h : \R \to \R$. In other words, the rescaling factor $\lambda(v) = \frac{h(|v|)}{|v|}$ depends only on the norm. Radial rescaling functions commute with orthogonal transformations, and in fact they are precisely the rescaling activations that do so (Lemma \ref{lem:rad-resc-orth}). 
 Examples include:
(1) Step-ReLU, discussed above. (2) The squashing function, where $ h(r) = \frac{r^2}{r^2 + 1}$. (3) Shifted ReLU, where  $ h(r) = \text{ReLU} (r - b )$ for $r >0$ and a real number $b$.  We refer to \cite{ganev_universal_2022, weiler_general_2019} and the references therein for more examples and discussion of radial functions. 

Let   $\quiver = (I, E)$ be a neural quiver.  We say that a $\quiver$-neural network $(\bfd, \bfW, \boldrho)$ is a {\it radial} $\quiver$-neural network if each $\rho_i$ is a radial rescaling activation. 
The following result is a straightforward consequence of the fact that $\rho_i$ commutes with orthogonal transformations, and implies that  one can simplify Algorithm \ref{alg:QR-rescaling} in the case of radial rescaling activations.

\begin{restatable}{prop}{propRadial}
\label{prop:rad-NNs}
	For radial $\quiver$-neural networks, Algorithm \ref{alg:QR-rescaling}  leaves the activation functions unchanged. 
\end{restatable}

\section{Projected gradient descent}\label{sec:proj-gd}

In practice, one typically applies a compression algorithm to a fully trained model in order to produce a smaller model that is more efficient at deployment. Some compression algorithms also accelerate training; this is the case, for example, when the compressed and original models have the same feedforward function after a step of gradient descent is applied to each. Unfortunately, for quiver neural networks with rescaling activations,  compression using Algorithm \ref{alg:QR-rescaling} before training does not yield the same result as training followed by compression (even in the basic sequential case with radial rescaling activations \cite{ganev_universal_2022}).  There is, however, an explicit mathematical relationship between optimization of the two  networks, assuming radial rescaling activations. Namely, as we discuss in this section, the loss of the compressed model after one step of gradient descent coincides with to the loss of a transformation of the original model after one step of {\it projected} gradient descent. We emphasize that the results of this section only hold for radial rescaling functions (which commute with orthogonal transformations), not general rescaling functions.  Proofs of the results of this section appear in Appendix \ref{app:proj-GD}. 

\subsection{Parameter space symmetries}\label{subsec:param-sym}
To state our results, we introduce additional notation.
Let $\quiver$ be a  quiver and $\bfd$ a dimension vector for $\quiver$. Set:
$$\Par(\quiver, \bfd) = \bigoplus_{e \in E} \R^{d_{t(e)} \times d_{s(e)}} \qquad \text{\rm and} \qquad O(\mathbf{d}^\text{\rm hid}) = \prod_{i \in I_{\text{\rm hidden}}} O(d_i)$$
to be the corresponding {\it space of trainable parameters}\footnote{A choice of trainable parameters is the same as a representation of the quiver, see Appendix \ref{app:quiver-reps}.} and {\it orthogonal symmetry group}, respectively. Note that  each tuple of weight matrices $\bfW$ in a $\quiver$-neural network belongs to $\Parqd$. 
An element of the orthogonal symmetry group group consists of the choice of an orthogonal transformation of  $\R^{d_i}$ for each hidden  vertex $i$, and results in a corresponding transformation of the weight matrices $W_e$ appearing in any element $\bfW$. To be explicit, a particular choice  $\mathbf{Q} = (Q_\inx)_{i}$ of orthogonal matrices results in the following linear transformation of weight matrices:
$$  \bfW\mapsto  \mathbf{Q} \cdot \bfW = 
\left( Q_{t(e)} \ \circ \ W_e  \ \circ \  Q_{s(e)}\inv \right)_e $$
Hence we obtain a linear action of $O(\mathbf{d}^\text{\rm hid}) $ on  on $\Parqd$. It is straightforward to show that, when using radial rescaling activations, this action leaves the feedforward function unchanged (analogous to the `positive scaling invariance' property of ReLU \cite{dinh_sharp_2017, neyshabur_path-sgd_2015, meng_g-sgd_2019}).

\subsection{Gradient descent maps}\label{subsec:pgd-set-up}
 Fix a neural quiver $\quiver = (I, E)$,  a dimension vector $\bfd$, and a tuple of radial rescaling activations $\boldrho = (\rho_i : \R^{d_i} \to \R^{d_i} )_{i \in I}$ (where, as usual, $\rho_i$ is the identity if $i$ is a source). For  any batch of training data $\{ (x_j, y_j)\} \subseteq \R^{d_\text{\rm in}} \times \R^{d_\text{\rm out}}$, we have  the loss function $\L : \Parqd  \to \R$ taking $\bfW$ to $\sum_j \mathcal C(F_{(\bfd, \bfW, \boldrho)}(x_j), y_j) $, where  $\mathcal C : \R^{d_\text{\rm out}} \times \R^{d_\text{\rm out}}\to \R$ is a cost function on the output space.  For a learning rate $\eta >0$, the corresponding  gradient descent map on $\Parqd$ is given by:
 \[ \gamma : \Parqd \to \Parqd \qquad \qquad \bfW \mapsto \bfW -  \eta  \nabla_\bfW \L.  \]
 Similarly, using the reduced dimension vector,  we have the loss function $\L_\text{\rm red} : \Parqdred  \to \R$ and gradient descent map $\gamma_{\text{\rm red}}$ on $\Parqdred$ given by $\L_\text{\rm red}(\bfX) $ $=$ $  \sum_j \mathcal C(F_{(\bfdred, \bfX, \boldrho\red)}  (x_j), y_j)$ and $\gamma_\text{\rm red} (\bfX) = \bfX - \eta \nabla_\bfX \L_\text{\rm red}$, respectively.

Next we define the projected gradient descent map. Fix a dimension vector $\bfd$ for the neural quiver $\quiver$.  
The {\it projected gradient descent} map on the parameter space $\Parqd$ is defined as:
\begin{align*}
\gamma_{\text{\rm proj}} : \Parqd & \to \Parqd      \qquad 
\bfW \mapsto  \Proj\left(  \bfW -  \eta \nabla_\bfW \L  \right)
\end{align*}
where  $\Proj : \Parqd \to \Parqd$ is the map that zeros out  all entries in the bottom left $(d_{t(e)} - \dred_{t(e)}) \times\dred_{s(e)}$ submatrix of each $W_e - \eta \nabla_{W_e} \L$. Schematically:
$$ (\bfW - \eta \nabla_{\bfW} \L)_e = W_e - \eta \nabla_{W_e} \L = \begin{bmatrix}
* & * \\
* & *
\end{bmatrix} \ \  \longmapsto  \ \  \Proj(\bfW - \eta \nabla_{\bfW} \L)_e = \begin{bmatrix}
* & * \\
0 & *
\end{bmatrix} $$
where we regard $\bfW -  \eta \nabla_\bfW \L$ as a tuple $(W_e - \eta \nabla_{W_e} \L)_{e \in E} \in \Parqd$.
Hence, while all entries of each matrix  $W_e$ in the tuple $\bfW$ contribute to the computation of the gradient $\nabla_\bfW \L$, only those not in the bottom left submatrix get updated under the projected gradient descent map $ \gamma_\text{\rm proj}$, while those in the bottom left submatrix are zeroed out. 

We are now ready to state the relationship between gradient descent on the compressed model and projected gradient descent on the original model. 
Let $\bfW$ be a tuple in $\Parqd$. Applying Algorithm \ref{alg:QR-rescaling}, let $\bfW\red \in \Parqdred$ be the reduced parameters and $\bfQ \in O(\bfdhid)$ the orthogonal symmetries. Set $\bfT = \bfQ\inv \cdot \bfW$ to be the transformed parameters, so that  $\bfW = \bfQ \cdot \bfT$.

\begin{restatable}{theorem}{thmProjGD}
\label{thm:qt-dim-red}
	Let $\bfW \in \Parqd$, and let $\bfW\red, \bfQ, \bfT$ be as above. For any $k \geq 0$, we  have\footnote{\label{footnote:inclusion}More precisely, the second equality  is  $\gamma_\text{\rm proj}^{k}( \bfT ) =   \iota(	\gamma_\text{\rm red}^k (\bfW\red) ) + \bfT - \iota(\bfW\red)$ where $\iota : \Parqdred \hookrightarrow \Parqd$ is the natural inclusion. See Appendix \ref{app:proj-GD}.}:
	$$\gamma^k( \bfW) = \bfQ \cdot \gamma^k(\bfT) \qquad \qquad  \gamma_\text{\rm proj}^k (\bfT) =   \gamma_{\text{\rm red}}^k (\bfW\red)   + (\bfT - \bfW\red).$$
\end{restatable}

\begin{figure}
	\centering
	     \includegraphics[width=0.8\textwidth]{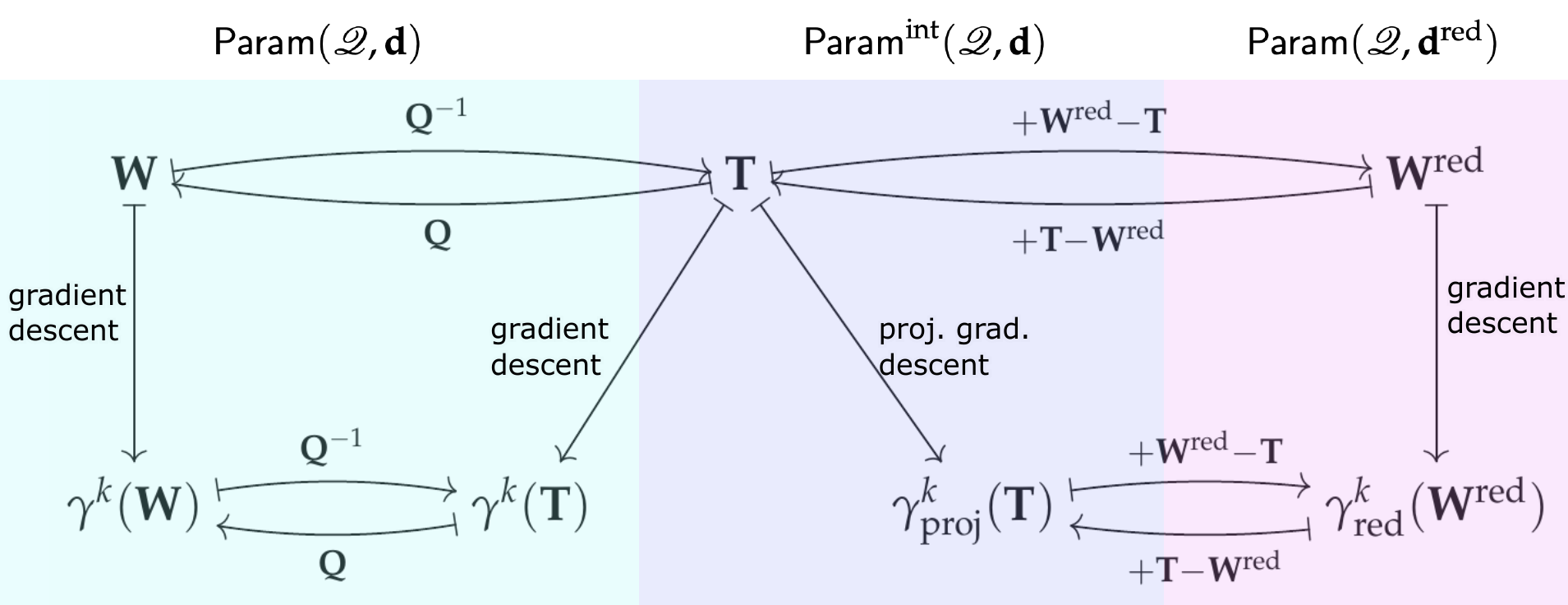}
	    \caption{Schematic summary of Theorem \ref{thm:qt-dim-red}. The left horizontal maps indicate the action of $\bfQ^{\pm 1}$, the right horizontal maps indicate translation by $\pm (\bfT - \bfW\red)$, the and the vertical maps are various versions of gradient descent. The colored regions indicate the (smallest) vector space to which the various  tuples naturally belong; the interpolating subspace $\Parintqd$ is defined in Appendix \ref{app:proj-GD}. }
	    \label{fig:proj-GD}
\end{figure}

Hence, the gradient descent optimization of the original model $\bfW$ and the transformed model $\bfT$ are equivalent, since at any stage they are related by the orthogonal change-of-basis action of $\bfQ$. Meanwhile, optimization of the transformed model $\bfT$ via projected gradient descent is equivalent to the gradient descent optimization of the compressed model $\bfW\red$, since at any stage one can add or subtract the quantity $\bfT - \bfW\red$ to obtain one from the other. Note that, for  $k \geq 0$, we have the $k$-fold composition $\gamma^k = \gamma \circ \gamma \circ \cdots \circ \gamma$ and similarly for  $\gamma_{\text{\rm red}}$ and $\gamma_{\text{\rm proj}}$. We summarize Theorem \ref{thm:qt-dim-red} in Figure \ref{fig:proj-GD} and give  a proof in Appendix \ref{app:proj-GD}. 

\section{Experiments}\label{sec:experiments}

We provide a software package making it easy to build quiver neural networks by simply specifying the quiver $\quiver$, the dimension vector, and the activation functions.  We also implement the model compression algorithm \ref{alg:QR-rescaling}.  As a check for our theoretical results we provide empirical verification of the above theorems.  
For simplicity, our experiments only consider the case of radial rescaling activations. 

\paragraph{\bf Empirical verification of Theorem \ref{thm:mod-comp-rescaling}}

We consider the quivers displayed in Figure \ref{fig:exp-quiver}. We write dimension vectors as a list according to alphabetic order of the vertices (which is also a topological order), that is, as $(d_a, d_b, d_c, \dots)$.
We select the following dimension vectors: $(2,4,8,2)$, $(1,2,8,2,6)$, and $(2, 4, 4, 8, 2)$, respectively. The corresponding reduced dimension vectors are: $(2,3,6,2)$, $(1,2,4,2,6)$, and $(2,3,3,7,2)$. The weights, biases, inputs, and labels are randomly assigned from $\mathrm{Uniform}([0,1))$.  Over 10 trials, the feedforward functions of the original and compressed models agree up to machine precision: $|F_\text{\rm orig} - F_\text{\rm compressed}| < 10^{-6}$. 

\paragraph{\bf Empirical verification of Theorem \ref{thm:qt-dim-red}}

We verify the theorem for the three quivers appearing in Figure \ref{fig:exp-quiver}. Over 10 trials, one step of gradient descent of the compressed model matches one step of projected gradient descent of the original model up to machine precision:
$$|\gamma(\bfW) - \bfQ \cdot \gamma(\bfT) | < 10^{-5} \qquad |\gamma_\text{\rm proj}(\bfT) + \bfW\red - \gamma_\text{\rm red}(\bfW\red) - \bfT | < 10^{-6}$$

\begin{figure}[t]
 \resizebox{0.3\textwidth}{!}{	\xymatrix{ \\ \stackrel{a}{\bullet}  \ar[r] \ar@/^1pc/[rr]  & \stackrel{b}{\bullet}  \ar[r] & \stackrel{c}{\bullet}  \ar[r]    & \stackrel{d}{\bullet}   \\ 
			&  &    {\color{\biasvercol} \bullet}  \ar[ru] \ar[u] \ar[lu]    &  &    }} 
   \qquad 
   \resizebox{0.15\textwidth}{!}{\xymatrix{
\stackrel{a}{\bullet} \ar[rd] & & \stackrel{d}{\bullet}  \\
& \stackrel{c}{\bullet}  \ar[ru] \ar[rd] & \\
\stackrel{b}{\bullet}  \ar[ru] & & \stackrel{e}{\bullet}  \\
&  {\color{\biasvercol} \bullet} \ar[uu] \ar[uuur] \ar[ur] 
    }} \qquad   
   \resizebox{0.21\textwidth}{!}{\xymatrix{
  & \stackrel{b}{\bullet} \ar[dr] &  \\
\stackrel{a}{\bullet}  \ar[ru] \ar[rd] & &  \stackrel{d}{\bullet} \ar[r] & \stackrel{e}{\bullet}  \\
& \stackrel{c}{\bullet}  \ar[ru] & & \\
&  {\color{\biasvercol} \bullet} \ar[u] \ar@/^1pc/[uuu]  \ar[uur] \ar[uurr]
    }}
\caption{Quivers used in the experiments of Section \ref{sec:experiments}. The alphabetical order of the vertices is a topological order. As usual, the bottom vertex is the bias vertex. }
\label{fig:exp-quiver}
\end{figure}
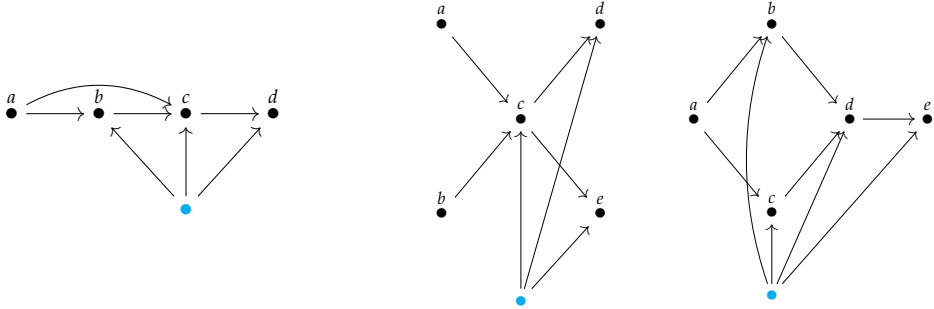

\section{Conclusions and future work}\label{sec:conclusion}

In this paper, we propose a new formalism for the study of neural network connectivity and data flows via the notion of a quiver neural network.  While many of our constructions originate in the somewhat abstract mathematical domain of quiver representation theory, we argue that there are practical advantages to quiver neural networks. Namely, our formalism captures connectivity architectures in the simplest possible way, emphasizes parameter space symmetries, leads to an explicit lossless model compression algorithm, and links the compression to gradient descent optimization. 

Moreover, our work opens a number of directions for future work. First, a modification of our techniques may lead to  a generalization of Algorithm \ref{alg:QR-rescaling} beyond rescaling activations to certain other families of non-pointwise activations. Second, we expect our framework to extend to encapsulate equivariant and convolutional neural networks. Third, our constructions are compatible with various regularization techniques (particularly $L_2$-regularization); we have not yet explored the practical consequences of this compatibility. Fourth, future work would include an evaluation of empirical performance on real-world data sets. Finally, a natural next theoretical step would be to explore relations with quiver varieties and reformulate optimization problems geometrically, with the aim to obtain better convexity properties of the appropriate version of the loss function.

\section*{Acknowledgements}

We are grateful for valuable feedback and suggestions from Avraham Aizenbud, Rustam Antia, Marco Antonio Armenta, Niklas Smedemark-Margulies, Jan-Willem van de Meent, Twan van Laarhoven, and Rose Yu.  This work was partially funded by the NWO under the CORTEX project (NWA.1160.18.316) and NSF grant \#2134178. Robin Walters is supported by the Roux Institute and the Harold Alfond Foundation.

\bibliography{main.bib}

\vfill
\pagebreak

\appendix
\section{Model compression}\label{app:model-compression}

In this appendix, we discuss model compression in more detail. We begin with a proof of Theorem \ref{thm:mod-comp-rescaling}, elaborate on an improvement and an alternative to Algorithm \ref{alg:QR-rescaling}, define subnetworks of quiver neural networks, and include results about radial neural networks. 

\subsection{Proof of Theorem \ref{thm:mod-comp-rescaling}}
We adopt the notation of Section \ref{sec:rescaling-MC} and Algorithm \ref{alg:QR-rescaling}. For each vertex $i$, let $\lambda_i : \R^{d_i} \to \R$  be the rescaling factor for the rescaling activation $\rho_i$. 

\begin{lemma}\label{lem:taui}
    For any vertex $i$:
\begin{enumerate}
        \item The function $\tau_i$ is a rescaling activation with rescaling factor equal to $\lambda_i \circ Q_i \circ \inc_i$.  
        \item The following identity holds: $Q_i \circ \inc_i \circ \tau_i = \rho_i \circ Q_i \circ \inc_i$.
\end{enumerate}
\end{lemma}

\begin{proof}
 Let $x \in \R^{\dred_i}$. To show the first claim, we compute:
\begin{align*}
\tau_i(x) & = \pi_i \circ Q_i\inv \circ \rho_i \circ Q_i \circ \inc_i(x) = \pi_i \circ Q_i\inv \left( \lambda_i(Q_i \circ \inc_i (x)) Q_i \circ 
\inc_i (x) \right) \\ 
&= \pi_i \left( \lambda_i(Q_i \circ \inc_i (x)) \inc_i(x) \right) = \lambda_i(Q_i \circ \inc_i(x)) x  
\end{align*}
where we use the fact that $Q_i$ is a linear map, and that the composition $\pi_i \circ \inc_i$ is the identity on $\R^{\dred_i}$. This proves the first claim. 
For the second claim, we compute:
\begin{align*}
Q_i \circ \inc_i \circ \tau_i (x)  & = Q_i \circ \inc_i  \left(  \lambda_i(Q_i \circ \inc_i(x)) x \right) = \lambda_i(Q_i \circ \inc_i(x))\left( Q_i \circ \inc_i (x) \right) =  \rho_i \circ Q_i \circ \inc_i (x)
\end{align*}
where we use the fact that $Q_i$ and $\inc_i$ are linear maps.
\end{proof}

We require some notation for the remainder of the proof. Let $F$ be the feedforward function of the input network $(\bfd, \bfW, \boldrho)$ and $F\red$ the feedforward function of the compressed network $(\bfdred, \bfW\red, \boldsymbol{\tau})$ produced by Algorithm \ref{alg:QR-rescaling}. Correspondingly, we have the partial feedforward functions $F_i$ and $F\red_i$ for every vertex $i$. Set $\dred_{\to i} = \sum_{e\in t\inv(i)} \dred_{s(e)}$ to be the sum of the reduced dimension vector values $\dred_j$ for vertices $j$ with an outgoing edge to vertex $i$.  Enumerate the incoming edges as $e_1, \dots, e_{|t\inv(i)|}$. 
Let  $M_i$ be as in the algorithm, so that:
$$M_i = \begin{bmatrix}
W_{e_1} \circ Q_{s(e_1)}  \circ \inc_{s(e_1)} & \cdots & W_{e_{|I|}} \circ Q_{s(e_{|I|})}  \circ \inc_{s(e_{|I|})} 
\end{bmatrix} \in \R^{d_i \times d_{\to i}}$$ 
For $x \in \R^{d_\text{\rm in}} $, set
$$F\red_{\to i}(x) = \begin{bmatrix}
F\red_{s(e_1)}(x) \\
\vdots \\
F\red_{s(e_{|t\inv(i)|})(x)}
\end{bmatrix} \in \R^{d_{\to i}}.$$
Next, we show that:
\begin{equation}\label{eqn:F=QincFred}
F_{i} =  Q_i \circ \inc_i \circ F\red_{i}
\end{equation}
for all vertices $i$, where $Q_i =  \id_{d_i}$ if $i$ is a source or sink. We proceed by induction. The base case is when $i$ is a source, and is easy. For the induction step, we take $x \in \R^{d_\text{\rm in}}$ and compute:
\begin{align*}
F_{i}(x)  &= \rho_i \left( \sum_{e \in t\inv(i)} W_e \circ F_{s(e)}(x)  \right)  =\rho_i \left( \sum_{e \in t\inv(i)}  W_e \circ Q_{s(e)}  \circ \inc_{s(e)}\circ  F\red_{s(e)}(x)  \right)  = \rho_i \circ  M_i \circ  F\red_{\to i}(x) \\ 
& = \rho_i \circ Q_i \circ \inc_i \circ R_i \circ  F\red_{\to i}(x)  = Q_i \circ \inc_i \circ \tau_i \circ R_i  \circ  F\red_{\to i}(x)   \\ 
&=  Q_i \circ \inc_i  \circ \tau_i    \left( \sum_{e \in t\inv(i)}  W\red_e  \circ F\red_{s(e)}(x) \right)=  Q_i \circ \inc_i \circ F\red_{i}(x)
\end{align*}
where the first equality follows from 	the definition of $F_i$; the second from the induction hypothesis; the third from  the definitions of $M_i$ and $F\red_{\to i}(x)$; the fourth from the definitions of $Q_i$ and $R_i$; the fifth from the Lemma \ref{lem:taui};  the sixth from the definition of $R_e$ and $F\red_{\to i}(x)$; and the last from the  definition of $F\red_i$. This establishes Equation \ref{eqn:F=QincFred}. The theorem now follows from the definition of the feedforward function, and the fact that $Q_i = \inc_i = \id_{d_i}$ if $i$ is a sink.

\subsection{Improvement to Algorithm \ref{alg:QR-rescaling}}\label{appsubsec:improvement}

As mentioned in Section \ref{sec:rescaling-MC}, our model compression algorithm can be improved by considering the ranks of the matrices $M_i$. We give this improvement in  Algorithm \ref{alg:QR-minimal}, displayed at the end of this appendix. For clarity, we describe the algorithm in more detail as follows. Set $k_i = d_i$ for each source and sink vertex $i$. Fix a topological order of the vertices, and, for each vertex that is neither a source nor a sink, do:
\begin{enumerate}
    \item  Define $\widetilde{M}_i$ using a merging procedure as before, except use the inclusion map $\inc_{k_j,d_j} :\R^{k_i} \hookrightarrow \R^{d_i}$ into the first $k_i$ coordinates instead of $\inc_{\dred_j, d_j}$. So  $\widetilde{M}_i$ is of size $d_i \times k_{\to i}$. We abbreviate $\inc_{k_j,d_j} $ by $\widetilde{\inc}_i$.
    
    \item  Let $k_i$ be the rank of $\widetilde{M}_i$. Observe that $k_i \leq \min(d_i, k_{\to i})$, so $k_i \leq \dred_i$. 
    
    \item  Set $P_i \in \R^{k_{\to i} \times k_{\to i}}$ to be a permutation matrix  such that the first $k_i$ columns of $\widetilde{M}_i$ are linearly independent. 
    
    \item  Let $\widetilde{Q}_i \widetilde{\inc}_{i} \widetilde{R}_i$ be the complete QR decomposition of $\widetilde{M}_i P_i$. Note that the choice of permutation $P_i$ implies that  $\widetilde{R}_i$ is full rank.
    
    \item Set $\widetilde{W}\red_{\to i}  = \widetilde{R}_i P_i\inv$, which is a full-rank matrix of size $k_i \times k_{\to i}$. Since $k_i \leq k_{\to i}$, we see that in fact $\widetilde{W}_{\to i}\red$ is surjective. 
    
    \item Extract the matrices $\widetilde{W}_e\red$ from $\widetilde{W}_{\to i}\red$ as in the original algorithm. 
    
    \item Update the activation $\rho_i$ to $\widetilde{\tau}_i = \widetilde{\pi}_i \circ \widetilde{Q}_i\inv \circ \rho_i \circ \widetilde{Q}_i \circ \widetilde{\inc}_i$ where $\widetilde{\pi}_i : \R^{d_i} \to \R^{k_i}$ is the projection onto the frist $k_i$ coordinates. 

\end{enumerate}
For sink vertices, one updates the weight matrix $W_e$ on each incoming edge to $W_e \circ \widetilde{Q}_{s(e)} \circ \widetilde{\inc}_{s(e)}$.
Using this version of the algorithm, the compressed network has dimension vector $\bfk = (k_i)$.  The methods {\tt hstack}, {\tt rank}, {\tt permutation}, and {\tt extract} of Algorithm \ref{alg:QR-minimal} reflect the procedures described in steps (1), (2), (3), and (6) above, respectively. 

Observe that Algorithm \ref{alg:QR-minimal} provides an actual improvement  only if some of the matrices $\widetilde{M}_i$ matrices are not of full rank (otherwise $k_i = \dred_i$ for all vertices $i$), which is a situation uncommon in practical applications. One can also replace each $\widetilde{M}_i$ by a version with singular values above a given threshold. In this case, the feedforward function of the compressed model would differ from that of the original model, so the compression would not be lossless.

Examination of Algorithm \ref{alg:QR-minimal} reveals that the image of the common feedforward functions of the compressed and original models lies in a subspace of $\R^{k_\text{\rm out}}  = \R^{d_\text{\rm out}}$ of dimension $\sum_{i \in \text{\rm sinks}} \mathrm{rank}(\widetilde{M}_i)$, where $\widetilde{M}_i$ is defined in the same way for sinks as it is for non-sinks. Hence, if desired, the basis of the output space can be changed to make $k_i$ equal to $\mathrm{rank}(\widetilde{M}_i)$ for output vertices. Similar remarks hold for Algorithm \ref{alg:QR-rescaling}, where $\dred_i = \min(d_i, \dred_{\to i})$ for output vertices if one allows changes of basis for the output space. 

\subsection{Alternative algorithm}\label{appsubsec:alternative}

A key tool in Algorithm \ref{alg:QR-rescaling} is the use of the QR decomposition to effectively change the basis of the feature spaces $\R^{d_i}$ at the hidden vertices. It is possible to preform this change-of-basis differently, as exhibited in Algorithm \ref{alg:dim-red-rescaling} (displayed at the end of this appendix). For clarity, we describe the algorithm in more detail as follows. For every source vertex, set $B_i$ to be the identity $d_i \times d_i$ matrix and set $k_i = d_i$. For each vertex that is neither a source nor a sink, do:
\begin{enumerate}
    \item Form the matrix $\hat B_i$ to be the $d_i \times k_{\to i}$ matrix formed by horizontally concatenating the matrices $W_e B_{s(e)}$ for all incoming edges to $i$. 
    \item Proceeding from left to write, check if each column of $\hat B_i$ is in the span of the preceeding columns, and if so, remove it. 
    \item Set $k_i$ to be the number of columns of $B_i$. (Equivalently, $k_i$ is  the rank of $\hat B_i$.)
    \item Observe that $B_i$ is injective, so there exists a matrix $C_i$ such that $C_i B_i$ is the identity $k_i \times k_i$ matrix. We call $C_i$ a left inverse of $B_i$, and note that it is not unique in general. An adequate choice can be easily computed. 
    \item Set $\tilde{\tau}_i$ to be $C_i \circ \rho_i \circ B_i$. 
    \item For each incoming edge $e$ to $i$, set $\tilde{W}_e\red$ to be the matrix product $C_i W_e B_{s(e)}$, which is a matrix of size $k_i \times k_{s(e)}$. 
 \end{enumerate}
For sink vertices, one sets $k_i = d_i$ and updates the weight matrix $W_e$ on each incoming edge to $W_e \circ B_{s(e)}$. The methods {\tt hstack}, {\tt full\_rank}, {\tt num\_col}, and {\tt left\_inverse} of Algorithm \ref{alg:dim-red-rescaling} reflect the procedures described in steps (1), (2), (3), and (4) above, respectively. 
It is easy to verify that, for any vertex $i$ and any edge $j \stackrel{e}{\rightarrow} i$, we have $B_i \circ \tau_i = \rho_i \circ B_i$ and $B_i \circ W_e\red = W_e \circ B_j$, 
where we use the notation of Algorithm \ref{alg:dim-red-rescaling}. Furthermore, we recover the  model compression of Algorithm \ref{alg:dim-red-rescaling} from Algorithm \ref{alg:QR-minimal} by setting  $B_i := \tilde{Q}_i \circ \tilde{\inc}_{i}$, $C_i = \tilde{\pi}_{i} \circ \tilde{Q}_i\inv$, and $\tau_i = \tilde{\pi}_{i} \circ  \tilde{Q}_i\inv \circ \rho_i  \circ \tilde{Q}_i \circ \tilde{\inc}_{i}$. 

\subsection{Subnetworks of quiver neural networks}\label{appsubsec:subnetworks}

We now introduce the notation of a subnetwork of a $\quiver$-neural network, and prove that the feedforward function of a subnetwork is intertwined with the feedforward function of the full network. Algorithms \ref{alg:QR-rescaling}, \ref{alg:QR-minimal}, and \ref{alg:dim-red-rescaling} each produce a subnetwork of the original input $\quiver$-neural network. In the case of Algorithms \ref{alg:QR-minimal} and \ref{alg:dim-red-rescaling}, the subnetworks are minimal in certain precise sense that we explain below. 

Let $\quiver$ be a neural quiver. A {\it subnetwork} of a $\quiver$-neural network $(\bfd, \bfW, \boldrho)$ consists of a $\quiver$-neural network $(\bfk, \bfV, \boldsymbol{\tau})$ together with an injective linear map $\alpha_i : \R^{k_i} \hookrightarrow \R^{d_i}$ for each vertex such that:
\[ W_e \circ \alpha_{s(e)} = V_e \circ \alpha_{t(e)} \qquad \mathrm{and} \qquad \rho_i \circ \alpha_i = \alpha_i \circ \tau_i\]
for each vertex $i$ and each edge $e$, and $\alpha_{i_\text{\rm bias}}$ is the identity for the bias vertex. We group the maps $\alpha_i$ into a tuple $\boldsymbol{\alpha} = (\alpha_i)$ and write $(\bfk, \bfV, \boldsymbol{\tau}) \stackrel{\boldsymbol{\alpha}}{\hookrightarrow} (\bfd, \bfW, \boldsymbol{\rho})$.  We see that $k_i \leq d_i$ for each vertex $i$.
The maps $\alpha_i$ define maps $\alpha_\text{\rm in} : \R^{k_\text{\rm in}} \hookrightarrow \R^{d_{\text{\rm in}}}$ and $\alpha_\text{\rm out} : \R^{k_\text{\rm out}} \hookrightarrow \R^{d_{\text{\rm out}}}$. 

\begin{prop}\label{prop:feedforwards}
    Let $(\bfk, \bfV, \boldsymbol{\tau}) \stackrel{\boldsymbol{\alpha}}{\hookrightarrow} (\bfd, \bfW, \boldsymbol{\rho})$ be a subnetwork, and let $\tilde{F}$ and $F$ be the respective feedforward functions. Then:
    \[ F \circ \alpha_\text{\rm in} = \alpha_\text{\rm out} \circ \tilde{F} \]
\end{prop}

\begin{proof}[Sketch of proof] It suffices to show that the partial feedforward functions satisfy $F_i \circ \alpha_\text{\rm in} = \alpha_i \circ \tilde{F_i}$
for all vertices $i$. To this end, we fix a topological order and proceed by induction. The base step is straightforward (using the fact that $\alpha_{i_\text{\rm bias}} = \id_\R$). The key computation in the induction step is:
\begin{align*}
F_{i} \circ \alpha_\text{\rm in}  &= \rho_i \left( \sum_{e \in t\inv(i)} W_e \circ F_{s(e)} \circ \alpha_\text{\rm in}  \right)  =\rho_i \left( \sum_{e \in t\inv(i)}  W_e \circ \alpha_i \circ   \tilde{F}_{s(e)}  \right)  = \rho_i \left( \sum_{e \in t\inv(i)} \alpha_i \circ V_{s(e)} \circ \tilde{F}_{s(e)} \right) \\ & = \rho_i \circ \alpha_i \left( \sum_{e \in t\inv(i)} V_{s(e)} \circ \tilde{F}_{s(e)} \right)   =  \alpha_i \circ \tau_i \left( \sum_{e \in t\inv(i)} V_{s(e)} \circ \tilde{F}_{s(e)} \right) =  \alpha_i  \circ \tilde{F}_{i}
\end{align*}
\end{proof}

\begin{prop}\label{prop:subnet}
    The outputs of Algorithms \ref{alg:QR-rescaling}, \ref{alg:QR-minimal}, and \ref{alg:dim-red-rescaling} each produce a subnetwork of the original input $\quiver$-neural network  with rescaling activations.
\end{prop}

We omit a full proof of this proposition, as it is straightforward: one uses the maps $\alpha_i = Q_i  \inc_i$,  $\alpha_i = \widetilde{Q}_i  \widetilde{\inc}_{i}$, and $\alpha_i = B_i$, respectively. Note that we recover Theorem \ref{thm:mod-comp-rescaling} from Propositions \ref{prop:feedforwards}  and \ref{prop:subnet} since $Q_i\inc_i$ is the identity for any source or sink. 

We say that  a subnetwork $(\bfk, \bfV, \boldsymbol{\tau}) \stackrel{\boldsymbol{\alpha}}{\hookrightarrow} (\bfd, \bfW, \boldsymbol{\rho})$  is {\it source-framed} if we have $k_i = d_i$ and $\alpha_i = \id_{d_i}$ for any source vertex $i$. 
We say that a subnetwork of a $\quiver$-neural network $(\bfd, \bfW, \boldsymbol{\rho})$ is {\it minimal} if the following two conditions are satisfied: (1) it is source-framed, and (2) it a subnetwork of any other source-framed subnetwork of $(\bfd, \bfW, \boldsymbol{\rho})$. In particular, the value of the dimension vector of any other subnetwork  is at least $k_i$ at each vertex $i$. 

\begin{prop}\label{prop:minimal-subnet}
    The outputs of Algorithms \ref{alg:QR-minimal} and \ref{alg:dim-red-rescaling} each produce the minimal subnetwork of the original input $\quiver$-neural network  with rescaling activations.
\end{prop}

\begin{proof}
    We give a proof only in the case of  Algorithm \ref{alg:QR-minimal}; the argument for Algorithm \ref{alg:dim-red-rescaling} is similar. Set $\alpha_i = \widetilde{Q}_i \widetilde{\inc}_{i}$. Suppose $(\bfn, \bfU, \boldsymbol{\nu}) \stackrel{\boldsymbol{\beta}}{\hookrightarrow} (\bfd, \bfW, \boldsymbol{\rho})$ is a source-framed subnetwork. For each vertex, we aim to  define an injective linear map
    $$\psi_i : \R^{k_i} \to \R^{n_i}$$
    such that $\beta_i \circ \psi_i = \alpha_i$. If $i$ is a source, set $\psi_i$ to be the identity. Proceeding by induction over a topological order of the vertices, take a non-source vertex $i$ and assume that we have defined $\psi_j$ for all vertices $j$ with an outgoing edge to $i$. For a fixed ordering of the edges, we form the matrix $W_{\to i} \in \R^{d_i \times d_{\to i}}$ by horizontally stacking the matrices $W_e$ for edges incoming to $i$. Similarly we have the matrices $U_{\to i} \in \R^{n_i \times n_{\to i}}$ and  $\tilde{W}\red_{\to i} \in \R^{k_i \times k_{\to i}}$. We also have the linear map $\alpha_{\to i} : \bigoplus_{ j \to i} \R^{k_j} \to \bigoplus_{ j \to i} \R^{d_j}$, defined as $\alpha_{\to i} = \bigoplus_{ j \to i} \alpha_j $. The maps $\beta_{\to i}$ and $\psi_{\to i}$ are defined analogously. We compute:
    \begin{align*}
   \alpha_i \circ \tilde{W}_{\to i}\red  & =   {W}_{\to i} \circ \alpha_{\to i}  =    W_{\to i} \circ \beta_{\to i} \circ \psi_{\to i} = \beta_{\to i} \circ U_{\to i} \circ \psi_{\to i} 
    \end{align*}
    where the first equality uses the fact that the $\alpha_j$ define subnetwork, the second equality uses the induction hypothesis, and the third uses the fact that the $\beta_j$ define a subnetwork. Recall from the discussion in Section \ref{appsubsec:improvement} that the matrix $\tilde{W}_{\to i}\red$ is surjective. It follows that the image of $\alpha_i$ is contained in the image of $\beta_i$. Hence, we can choose a map $\psi_i : \R^{k_i} \to \R^{n_i}$ such that $\beta_i \circ \psi_i = \alpha_i$. Since $\alpha_i$ is injective, so is $\psi_i$. 
    
    Now that the maps $\psi_i$ have been defined, it is straightforward to show that $\psi_{t(e)} \circ \tilde{W}_e\red \circ \psi_{s(e)}$ for any edge $e$, and that $\psi_i \circ \tau_i = \nu_i \circ \psi_i$ for any vertex $i$. Thus, $(\bfk, \tilde{\bfW}\red, \boldsymbol{\tau})$ is a subnetwork of $(\bfn, \bfU, \boldsymbol{\nu})$ via the maps $\psi_i$. 
\end{proof}

\subsection{Radial rescaling functions}
In this section, we turn our attention to radial rescaling activations and radial neural networks. We first prove the following basic fact about radial rescaling functions:

\begin{lemma}\label{lem:rad-resc-orth}
	A rescaling function commutes with orthogonal transformations if and only if it is a radial  rescaling function.  
\end{lemma}

\begin{proof}
    It is straightforward to show that any radial rescaling function commutes with orthogonal transformations. For the opposite direction, suppose  $\rho$ is a rescaling activation on $\R^d$ with rescaling scalar-valued function $\lambda : \R^d \to \R$, and suppose $\rho$ commutes with orthogonal transformations. Then, one easily shows that $\lambda(Qv) = \lambda(v)$ for any $v \in \R^d$ and any $Q \in O(d)$. Since any two elements of $\R^d$ of the same norm are related by an orthogonal transformation, it follows that the rescaling factor $\lambda(v)$ depends only on the norm of $v$.
\end{proof}

Next, we restate and prove Proposition \ref{prop:rad-NNs}.

\propRadial*

\begin{proof} The fact that each activation $\rho_i$ commutes with orthogonal transformations implies that $\Q_i\inv \circ \rho_i \circ Q_i = \rho_i$. Hence, in Algorithm \ref{alg:QR-rescaling}, we have $\tau_i \gets \pi_i \circ \rho_i \circ \inc_i$ for any hidden vertex $i$. The image of $\R^{\dred_i} $ under $\rho_i$ lies in $\R^{\dred_i}$, so $\tau_i$ is nothing more than the restriction of $\rho_i$ to $\R^{\dred_i}$. 
\end{proof}

Finally, we consider subnetworks of radial $\quiver$-neural networks. Let $\quiver$ be a neural quiver. A {\it radial subnetwork} of a radial $\quiver$-neural network $(\bfd, \bfW, \boldrho)$ consists of a radial $\quiver$-neural network $(\bfk, \bfV, \boldsymbol{\tau})$ together with a linear isometry $\alpha_i : \R^{k_i} \hookrightarrow \R^{d_i}$ for each vertex such that: $ W_e \circ \alpha_{s(e)} = V_e \circ \alpha_{t(e)}$ and $\rho_i \circ \alpha_i = \alpha_i \circ \tau_i$
for each vertex $i$ and each edge $e$. Recall that the condition for $\alpha_i$ to be an isometry is that it is norm-preserving: $|\alpha_i(x)| = |x|$ for each $x \in \R^{k_i}$. Note that any linear isometry is injective. The analogues of Propositions \ref{prop:subnet} and \ref{prop:minimal-subnet} hold for radial rescaling activations, and are proven in essentially the same way.

\vfill
\pagebreak

\setcounter{algocf}{0}

\begin{algorithm}[t]

\SetKwFunction{QRdecompCom}{QR-decomp}
	\SetKwFunction{QRdecompRed}{QR-decomp}
	\SetKwInOut{Input}{input}
	\SetKwInOut{Output}{output}
	\SetKwInOut{Initialize}{initialize}
	\DontPrintSemicolon
	
	\Input {$\quiver$-neural network $(\bfd, \bfW, \boldrho)$ with each $\rho_i$ rescaling}
	\Output {$\quiver$-neural network $(\bfk, \tilde{\bfW}\red, \tilde{\boldsymbol{\tau}})$,   orthogonal matrices $\tilde{\bfQ} = (Q_i \in O(d_i))_{i \in I_\text{\rm hidden}}$ }

\BlankLine
	
 \For(\tcp*[r]{vertices are in topological order\vspace{-\baselineskip}}) 	{\text{\rm $i$ in $I$}    }{   
		\uIf{\text{\rm $i$ is not a source or sink}}{
		
		$\widetilde{M}_i \gets \mathtt{hstack}\left(W_e  Q_{s(e)}  \tilde{\inc}_{s(e)} : t(e) =i\right)$ \tcp*[r]{concatenate matrices}
		
		$k_i \gets \mathtt{rank}(\widetilde{M}_i)$
		
		$P_i \gets \mathtt{permutation}(\widetilde{M}_i, k_i)$ \tcp*[r]{permutation matrix}
		
		$\tilde{Q}_\inx, \tilde{R}_\inx \gets $ \QRdecompCom{$\widetilde{M}_i P_i$, \ $\mathtt{mode = `complete'}$ } \tcp*[r]{$M_{i} P_i = \tilde{Q}_\inx  \widetilde{\inc}_\inx  \tilde{R}_\inx$}
		
		\For{$e$ \text{\rm such that} $t(e) = i$}{
			$\tilde{W}_e\red \gets \mathtt{extract}\left( \tilde{R}_i P_i\inv , e \right)$ \tcp*[r]{extract the reduced weights}
			}

		$\tilde{\tau}_i \gets \tilde{\pi}_i \circ \tilde{Q}_i\inv \circ \rho_i \circ \tilde{Q}_i \circ \tilde{\inc}_i$  \tcp*[r]{update the activations}		
		}
		
		\uElseIf{$i$ \text{\rm is a sink}}{
        $\tilde{\tau}_i \gets \rho_i$
  
				\For
		{$e$ \text{\rm such that} $t(e) = i$}{
		
			$\tilde{W}_e\red \gets W_e  \tilde{Q}_{s(e)}  \tilde{\inc}_{s(e)}$ \tcp*[r]{extract the reduced weights}
			}
		}	
	}	
	
	\KwRet $(\bfk, \tilde{\bfW}\red, \tilde{\boldsymbol{\tau}})$, and  $\tilde{\bfQ}$
	\caption{QR Model Compression for Rescaling activations (\texttt{QR-Compress})}

\SetKwInOut{Input}{input}

\caption{Improved model compression for rescaling activations (see Section \ref{appsubsec:improvement})}
	\label{alg:QR-minimal}
\end{algorithm}

\begin{algorithm}[t]
	\SetKwInOut{Input}{input}
	\SetKwInOut{Output}{output}
	\SetKwInOut{Initialize}{initialize}
	\DontPrintSemicolon
	
	\Input {$\quiver$-neural network $(\bfd, \bfW, \boldrho)$ with each $\rho_i$ rescaling}
	\Output {$\quiver$-neural network $(\bfk, \tilde{\bfW}\red, \tilde{\boldsymbol{\tau}})$} 
		
	\For(\tcp*[r]{vertices are in topological order \vspace{-\baselineskip}}) 	{\text{\rm $i$ in the set of vertices}    }{   
		
		\uIf{\text{\rm $i$ is a source}}{
			$B_i \gets \id_{d_i}$\;
				}

    \uElseIf{\text{\rm $i$ is a sink}}{  
    $C_i \gets \id_{d_i}$

    }
		\uElse{
        $\hat B_i \gets\mathtt{hstack}\left( W_e B_{s(e)} \ : \ t(e) = i\right) $ \tcp*[r]{concatenate matrices horizontally}

        $B_i \gets \mathtt{full\_rank}(\hat B_i)$  \tcp*[r]{$B_i$ is always full rank}

        $k_i \gets \mathtt{num\_col}(B_i)$ \tcp*[r]{$k_i \leq d_i$}

        $C_i \gets \mathtt{left\_inverse}(B_i)$ \tcp*[r]{$C_i \in \R^{k_i \times d_i}$ exists since $B_i$ is full rank}

        $\tilde{\tau}_i \gets C_i \circ \rho_i \circ B_i$ \tcp*[r]{update rescaling activations}
  
			}  

		\For(\tcp*[r]{iterate through incoming edges (if any)\vspace{-\baselineskip}})
		{$e$ \text{\rm such that} $t(e) = i$}{
			$\tilde{W}\red_e  \gets C_i W_e B_{s(e)} \in \R^{k_i \times k_{s(e)}}$ \tcp*[r]{define reduced weights}
		}

	}
	
	\KwRet $(\bfk, \tilde{\bfW}\red,  \tilde{\boldsymbol{\tau}})$
	\caption{Alternative model compression for rescaling activations (see Section \ref{appsubsec:alternative}) }
		\label{alg:dim-red-rescaling}
\end{algorithm}

\pagebreak

\section{Projected gradient descent}\label{app:proj-GD}

In this appendix, we collect results related to projected gradient descent and provide a proof of Theorem \ref{thm:qt-dim-red}. We begin by introducing a subspace of $\Parqd$ that interpolates between $\Parqd$ and $\Parqdred$.
Fix a neural quiver $\quiver$ and a dimension vector $\bfd$ for $\quiver$. The {\it interpolating space}  $\Parintqd $ is defined as the subspace of $\Parqd$ consisting of those $\mathbf{T} = (T_e)_{e \in E} \in \Parqd$  such that, for each edge, the bottom left $(d_{t(e)} - \dred_{t(e)}) \times \dred_{s(e)}$ block of $T_e$ is zero. The following lemma is a straightforward consequence of Algorithm \ref{alg:QR-rescaling}.  

\begin{lemma}
	Let $\bfW \in \Parqd$ and let  $\bfQ \in O(\bfdhid)$ be the parameter symmetry in  produced by Algorithm \ref{alg:QR-rescaling}. Then $\bfT = \bfQ\inv \cdot \bfW$ belongs to the interpolating space $\Parintqd$.
\end{lemma}

Recall that $\inc_i :\R^{\dred_i} \hookrightarrow \R^{d_i}$ denotes the inclusion  into the first $\dred_i$ coordinates.   The proof of the next lemma is an elementary verification. 

\begin{lemma} Let $\bfT\in \Parintqd$. For each edge $e$, there is a matrix $\overline{T}_e  \in  \R^{\dred_{t(e)} \times \dred_{s(e)}}$  such that:
	$T_e \circ \inc_{s(e)} = \inc_{t(e)} \circ \overline{T}_e$.
	In particular,  the image of $\R^{\dred_{s(e)}}$ under $T_e$ is contained in $\R^{\dred_{t(e)}}$, that is:  $T_e\left(\R^{\dred_{s(e)}}\right) \subseteq \R^{\dred_{t(e)}}$. 
\end{lemma}

Hence, each  $T_e$ is block upper triangular with $\overline{T}_e$ appearing as the top left block: $ T_e = \begin{bmatrix} \overline{T}_e & * \\ 0 & * \end{bmatrix}.$  Moreover, the tuple $(\overline{T}_e)_{e\in E}$ belongs to $\Parqdred$. To explain the sense in which $ \Parintqd $ interpolates between $\Parqd$ and $\Parqdred$, consider the following diagram
\[
\xymatrix{ \Parqdred  \ar@/^/[rr]^{\iota_2} & & \Parintqd \ar@/^/[ll]^{q_2}  \ar@/^/[rr]^{\iota_1}  & & \Parqd \ar@/^/[ll]^{q_1} 
}
\]

\begin{itemize}
	\setlength\itemsep{10pt}
	\item  The map  $\iota_1$ is the natural inclusion. 
	\item The  projection $ q_1  : \Parqd  \to \Parintqd $ takes a tuple  $\mathbf{W} \in \Parqd$ and zeros to lower left $(d_i - \dred_i) \times \dred_{j}$ block of each $W_e$. Note that $\Proj = \iota_1 \circ q_1$, and the transpose of $q_1$ is the inclusion $\iota_1$.	
	\item The inclusion $\iota_2 : \Parqdred \hookrightarrow  \Parintqd$ takes $\bfX$ to the tuple whose matrix corresponding to $e$ is obtained from $X_e$ by padding with $d_{t(e)} - \dred_{t(e)}$ rows of zeros and $d_{s(e)} - \dred_{s(e)}$ columns of zeros: 
	$\iota_2(\bfX)_e = \begin{bmatrix} X_e & 0 \\ 0 & 0 	\end{bmatrix}$.	
	\item The projection  $q_2 : \Parintqd \twoheadrightarrow \Parqdred $  takes a tuple $\mathbf{T}$ in the interpolating space, and extracts the top left $\dred_{t(e)} \times \dred_{s(e)}$ block of each matrix $T_e$. In other words, $q_2(\bfT) = \overline{\bfT}$.  The transpose of $q_2$ is the inclusion $\iota_2$. 
\end{itemize}

Finally, we set $\iota = \iota_1 \circ \iota_2 : \Parqdred \hookrightarrow \Parqd$ to be the inclusion  of $\Parqdred$ into $\Parqd$. This is defined in essentially the same way as $\iota_2$.
To state the next result, recall the set-up of Section \ref{subsec:pgd-set-up}. Namely, fix a dimension vector $\bfd$ and radial activation functions $\boldrho = (\rho_i : \R^{d_i} \to \R^{d_i} )_{i \in I}$. For  any batch of training data, we have the loss functions $\L : \Parqd \to \R$ and $ \L_\text{\rm red} : \Parqdred  \to \R$.

\begin{lemma}\label{lem:repint} We have the following:
	\begin{enumerate}
	\item The inclusion  $\iota = \iota_1 \circ \iota_2 : \Parqdred \hookrightarrow \Parqd$ intertwines the loss functions; that is, we have: $\L \circ \iota = \L_\text{\rm red}$. 		
	\item We have that $\L \circ \iota_1 = \L_\mathrm{red} \circ q_2$. In other words, the following diagram commutes:
		\[\xymatrix{
			\Parintqd \ar@{->>}[rr]^{q_2} \ar@{_{(}->}[d]_{\iota_1} & & \Parqdred \ar[d]^{\L_\text{\rm red}} \\
			\Parqd \ar[rr]^\L & & \R
		}
		\]	
	\item 	For any $\mathbf{T} \in \Parintqd$, we have:
		$ q_1 \left( \nabla_{\iota_1(\mathbf{T})}  \L \right) = \iota_2 \left( \nabla_{q_2(\mathbf{T})} \L_\text{\rm red}  \right). $
	\end{enumerate}
\end{lemma}

\begin{proof}
	We begin with some set-up. We fix a topological order of the vertices of $\quiver$. For each vertex $i$, we set $\pi_i : \R^{d_i} \to \R^{\dred_i}$ to be the projection into the first $\dred_i$ coordinates. Observe that, if $\bfX \in \Parqdred$, then  $\iota(\bfX)_e = \inc_{t(e)} \circ X_e \circ \pi_{s(e)}$  for each edge $e$.
	For the first claim, it suffices to verify that,  for any $\bfX$ in $\Parqdred$, the feedforward  functions  of $\bfX$ and $\iota(\bfX)$ coincide. To this end, for any vertex $i$, we set  $F_i$ to be the $i$-th partial feedforward function of the $\quiver$-neural network $(\bfd, \iota(\bfX), \boldrho)$ and $G_i$ to be the $i$-th partial feedforward function of the $\quiver$-neural network $(\bfdred, \bfX, \boldrho\red)$. We  prove by induction that the following identity holds:
	\begin{equation} \label{eqn:fi=incfi} F_{i} = \inc_i \circ G_{i }\end{equation}
	for any vertex $i$.
	Indeed, the identity is true for all source vertices since $d_i = \dred_i$ for such vertices. For the induction step, we compute:
	\begin{align*}
	F_{i} &= \rho_i \left( \sum_{e \in t\inv(i)} \iota(\bfX)_e \circ F_{s(e)} \right)  =\rho_i \left( \sum_{e \in t\inv(i)} \inc_{i} \circ X_e \circ \pi_{s(e)} \circ \inc_{s(e)} \circ G_{s(e)}\right)\\
	& = \rho_i \circ \inc_i \left( \sum_{e \in t\inv(i)} X_e  \circ G_{s(e)}\right)  =  \inc_i \circ \rho\red_i \left( \sum_{e \in t\inv(i)} X_e  \circ G_{s(e) }\right)  =  \inc_i \circ G_{i}
	\end{align*}
	where the first equality follows from 	the definition of the neural function of $(\iota(\bfX), \bfa)$, the second from the induction hypothesis and the fact that $\iota(\bfX)_e = \inc_{t(e)} \circ X_e \circ \pi_{s(e)}$ for each edge $e$, the third from the linearity of $\inc_i$ and the fact that $\pi_i \circ \inc_i = \id_{\dred_i}$ for each vertex $i$, the fourth from the commutativity property of radial functions with inclusions (Lemma \ref{lem:rad-resc-orth}), and the last from the  definition of the neural function of $(\bfX, \bfa\vert_\bfX)$. 
	The second claim follows from an argument similar to the one used to prove the first claim, using the fact that $T_e \circ \inc_{s(e)} = \inc_{t(e)} \circ \overline{T}_e$ and $\rho_i \circ \inc_i = \inc_c \circ \rho\red_i$.  	The proof of the last claim is a straightforward computation (omitted) involving the commutative diagram appearing in the second claim. 
\end{proof}

We now restate Theorem \ref{thm:qt-dim-red} and give a proof. 

\thmProjGD*

\begin{proof} 
	The action of $\bfQ$ on $\Parqd$ is an orthogonal transformation, and does not change the feedforward function. Hence the first equality in the statement of the theorem follows from a basic interaction of orthogonal transformations with gradient descent   (see Proposition 2.5  of   \cite{ganev_universal_2022}).  
	For the second equality of the theorem, we proceed by induction. As noted in Footnote \ref{footnote:inclusion}, the second equality is actually  $\gamma_\text{\rm proj}^{k}( \bfT ) =   \iota(	\gamma_\text{\rm red}^k (\bfV) ) + \bfT - \iota(\bfV).$ 
	The base case $k=0$ is immediate. For the induction step, let $k >0$ and  set
	$$\mathbf{Z}^{(k)}  := \iota(	\gamma_\text{\rm red}^k (\bfV) ) + \bfT - \iota(\bfV).$$
	Each $\mathbf{Z}^{(k)}$ belongs to $\Parintqd$, so  $i_1( \mathbf{Z}^{(k)} ) = \mathbf{Z}^{(k)}$. Moreover,  $\overline{ \mathbf{Z}^{(k)} } = q_2 \left( \mathbf{Z}^{(k)}  \right) = \gamma_\text{\rm red}^k (\bfV) $.
	We compute:
	\begin{align*}
	\gamma_\text{\rm proj}^{k+1}( \bfT) & = \gamma_\text{\rm proj}\left(   \gamma_\text{\rm proj}^k(\bfT) \right)  =   \gamma_\text{\rm proj} \left(  \iota(  \gamma_\text{\rm red}^k(\bfV)) + \bfT - \iota(\bfV) \right) \\
	&=   \iota(  \gamma_\text{\rm red}^k(\bfV)) +  \bfT - \iota(\bfV) -  \Proj  \left ( \nabla_{ \iota(  \gamma_\text{\rm red}^k(\bfV)) + \bfT - \iota(\bfV) } \L    \right)\\
	&= \iota(  \gamma_\text{\rm red}^k(\bfV))  -    \iota_1 \circ q_1 \left( \nabla_{ \iota_1(\mathbf{Z}^{(k)}) } \L     \right)  +  \bfT - \iota(\bfV)\\
	&= \iota(  \gamma_\text{\rm red}^k(\bfV)) -    \iota_1 \circ \iota_2  \left( \nabla_{ q_2(\mathbf{Z}^{(k)}) } \L_\text{\rm red}  \right) +  \bfT - \iota(\bfV) \\
	&=  \iota \left( \gamma_\text{\rm red}^k(\bfV) -  \nabla_{  \gamma_\text{\rm red}^k (\bfV)  }\L_\text{\rm red} \right ) + \bfT - \iota(\bfV) = \iota \left(  \gamma_\text{\rm red}^{k+1}(\bfV)   \right) +  \bfT - \iota(\bfV)
	\end{align*}
	The second equality invokes the induction hypothesis, the third equality uses the definition of the projected gradient descent map ${\gamma}_\text{\rm proj}$, the fourth equality  relies on   the interaction between the gradient and orthogonal transformations (see \cite[Proposition 2.5]{ganev_universal_2022}), the fifth  and sixth equalities   follow from  Lemma \ref{lem:repint} above, and the last equality uses the definition of the gradient descent map $ \gamma_\text{\rm red}$. 	
\end{proof}

\begin{rmk} We note that the proof of Theorem \ref{thm:qt-dim-red} is parallels that of Theorem 7 of  \cite{ganev_universal_2022}. The main differences are (1) the different structure of $\Parqdred$ as a subspace of $\Parqd$, and (2) the recursive arguments related to a topological order of the vertices.  \end{rmk}

\vfill
\pagebreak

\section{The QR decomposition for quiver representations}\label{app:quiver-reps}

In this appendix, we summarize the mathematical formalism of quiver representations, and its relation to the results of this paper. We then prove a theoretical result on an analogue of the QR decomposition for representations of an acylic quiver. 

\subsection{Basic definitions}

Let $\quiver = (I, E)$ be a quiver and $\bfd$ a dimension vector for $\quiver$, i.e., an assignment of a non-negative integer $d_i$ to each vertex.
A {\it representation} of  $\quiver$ consists of  a matrix  for each edge, where the matrix $A_e$ corresponding to an edge  $ \stackrel{s}{\bullet} \stackrel{e}{\longrightarrow} \stackrel{t}{\bullet}$ must be of size $d_t \times d_s$. Therefore, a representation can be regarded as a tuple of matrices $\mathbf{A} = (A_e )_{e \in E}$ indexed by the set of edges, where $A_e \in \R^{d_{t(e)} \times d_{s(e)}}$. 
The set  $\Rep(\quiver, \mathbf{d})$ of  all possible representations of a quiver $\quiver$ with dimension vector $\mathbf{d}$  is the direct sum of matrix spaces,
and hence a vector space:  \[ \Rep(\quiver, \mathbf{d}) = \bigoplus_{e \in E}   \R^{d_{t(e)} \times d_{s(e)}}. \]
We see that the spaces $\Parqd$ and $\Rep(\quiver, \bfd)$ are the same, so (if $\quiver$ is a neural quiver) the space of parameters for $\quiver$-neural networks with widths $\bfd$ is the same as the space of representations of $\quiver$ with dimension vector $\bfd$. The former notation emphasizes the relation with neural networks and their parameters, while the latter emphasizes the relation with representation theory.
Let $\bfA$ be a representation of $\quiver$ of dimension vector $\bfd$. A {\it subrepresentation} of $\bfA$ is a representation $\bfB = (B_e)_{e \in E}$ with dimension vector $\bfk = (k_i)$, together with  an injective map $\alpha_i : \R^{k_i} \hookrightarrow \R^{d_i}$ for each vertex such that, for each edge $e$, we have: $A_e \circ \alpha_{s(e)} = \alpha_{t(e)} \circ B_e$.
In particular, $k_i \leq d_i$ for each vertex $i$.

A representation of a quiver may be viewed as a tuple of  matrices indexed by vertices rather than by edges, as we now explain. First, given a dimension vector $\bfd$ for $\quiver$, recall the {\it incoming dimension} $d_{\to i}$ at a vertex $i$  to be the sum of all the dimension vectors at vertices with an edge to $i$. In symbols, $d_{\to i} := \sum_{j \to i} d_j$, where the sum is over  the subset $s(t\inv(i))$ of  $I$. By convention, $d_{\to i} = 0$ if $i$ is a source.
Fix  an enumeration $E= \{e_1, \dots, e_{|E|}\}$ of the edges of $\quiver$.
 Let $\bfA$ be a representation of $\quiver$ with dimension vector $\bfd$. For each non-source vertex $i$, set:
\begin{equation}\label{eqn:index-vertices}
A_{\to i} \ = \  \left[  A_{e_{k_1}} \ \dots \ A_{e_{k_r}}   \right] \in \R^{d_{ i}  \times d_{\to i}} 
\end{equation} 
where  $\{e_{k_1}, \dots, e_{k_r}\} = t\inv(i)$ is the set of edges with target $i$, with $k_1  \leq \cdots \leq k_r$ (using the enumeration of the edges fixed above). Hence we obtain a map $\Rep(\quiver, \mathbf{d}) \stackrel{\sim}{\longrightarrow} \bigoplus_{i \in I} \R^{d_{i} \times d_{\to i}}$
taking a representation  $\bfA$ to the tuple\footnote{Technically, the index runs over non-source vertices in $I$, since  $A_i$ is not defined if $i$ is a source.}   $(A_{\to i})_{i \in I}$. This map is in fact an isomorphism. 

\subsection{The QR decomposition}
We are now ready to  state the QR-decomposition result for quiver representations. Our construction is very general; however, in its simplest form, the QR decomposition for quiver representations is not aligned with the practical purpose of model compression, and hence its appearance in the appendix rather than the main text. The two procedures are nonetheless closely related, and we discuss these connections.

Let $\quiver= (I, E)$ be a quiver, and let $\bfd$ be a dimension vector for $\quiver$. Recall from Section \ref{subsec:param-sym} the product of orthogonal groups $O(\bfdhid) = \prod O(d_i)$ where the product runs over all hidden vertices (i.e., vertices that are neither sources nor sinks). This group acts on $\Parqd = \Rep(\quiver, \bfd)$ by orthogonal change-of-basis transformations. 

\begin{prop}\label{prop:qr-quiver}
	Let $\quiver$ be an acyclic quiver with no double edges. Let $\bfA$ be a representation of $\quiver$ with dimension vector $\bfd$. Then there exist
	$\bfQ = (Q_i)_{i \in I} \in O(\bfdhid)$ and $\bfR  = (R_e)_{e \in E} \in \Rep(\quiver, \bfd)$
	such that $$\bfA = \bfQ \cdot \bfR,$$
	and, moreover, for each non-source and non-sink vertex $i$, the matrix $R_{\to i}$ is upper-triangular.
	\end{prop}

\begin{algorithm}[t]
	\SetKwFunction{QRdecompCom}{QR-decomp}
	\SetKwInOut{Input}{input}
	\SetKwInOut{Output}{output}
	\DontPrintSemicolon
	
	\Input{$\bfA = ( A_e \in \mathbb{R}^{d_{t(e)} \times d_{s(e) }})  \in \Rep(\quiver, \bfd) $}
	\Output{$\mathbf{Q} \in O(\bfdhid) \quad  ,  \quad \mathbf{R} \in \Rep(\quiver, \bfd)$}

	$\bfQ, \bfR \gets [\ ], [\ ]$  	\tcp*[r]{Initialize output matrix lists}

	\For(\tcp*[r]{Iterate through the vertices\vspace{-\baselineskip}})
	{$i$ \text{\rm in} $I$  }{   

		\uIf{$i$ \text{\rm is not a source or a sink}}{
			$M = \texttt{hstack}(A_{e} Q_{s(e)} : t(e)  = i )$ \tcp*[r]{Form the matrix $M =A_{\to i}$}
			$Q_i , R_{\to i} \gets $ \QRdecompCom{$M$} \tcp*[r]{QR decomposition}
			Append $Q_i$ to $\bfQ$\;
			\For(\tcp*[r]{Iterate through the incoming edges\vspace{-\baselineskip}})
			{\text{\rm $e$ such that $t(e)  = i$} }
			{$R_e = \texttt{extract}( R_{\to i}, e)$ \tcp*[r]{Extract $R_e$ from the matrix $R_{\to i}$	}
						Append $R_e$ to $\bfR$\;
			}
		}
		
		\ElseIf{$i$ \text{\rm is a sink}}{
			
				\For(\tcp*[r]{Iterate through the incoming edges\vspace{-\baselineskip}})
			{\text{\rm $e$ such that $t(e)  = i$} }
			{
				Append $R_e \gets A_e Q_{s(e)}$ to $\bfR$  \tcp*[r]{Define $R_e$}
			}
		}
	}
	
	\KwRet: $ \bfQ, \bfR$ 
	\caption{QR decomposition for quiver representations.}
	\label{alg:QR-decomp-quiver}
\end{algorithm}

\begin{proof}
The proof is constructive, based on Algorithm \ref{alg:QR-decomp-quiver}. The algorithm takes as input a representation $\bfA$ of dimension vector $\bfd$ and outputs $\bfQ, \bfR = \texttt{QRDecomp}(\bfA)$, where $\bfQ \in O(\bfdhid)$ is an element of the change-of-basis symmetry group and $\bfR \in \Rep(\quiver, \bfd)$ is a representation of $\quiver$ with dimension vector $\bfd$. 
To see that  $\bfQ$ and $\bfR$ have the desired properties, first note that $R_{\to i}$ is indeed upper-triangular for every non-sink vertex $i$, and $Q_i$ is indeed orthogonal for every $i$. Next,  the defining property of  the topological ordering implies that, by the time the algorithm reaches a vertex $i$, the matrix $A_{\to i}$ has been updated to $A_{\to i}^\prime = A_{\to i} \circ \text{\rm Diag}\left(Q_{s(e_{k_1})}  \ , \  Q_{s(e_{k_2})} \  , \  \dots \ , \  Q_{s(e_{k_r})}\right)$, where $t\inv(i) =\{e_{k_1}, \dots, e_{k_r}\}$,  with $k_1 \leq \cdots \leq k_r$. If $i$ is not a sink, the algorithm computes the QR decomposition of this matrix $A_{\to i}^\prime$ to obtain $Q_i$ and $R_{\to i}$. Hence:
	\begin{align*}
	(\bfQ \cdot \bfR)_{\to i} & =   Q_i  \circ R_{\to i}  \circ \text{\rm Diag}\left(Q_{s(e_{k_1})}\inv  \  \dots \ , \  Q_{s(e_{k_r})}\inv\right)  = A_{\to i}^\prime \circ \text{\rm Diag}\left(Q_{s(e_{k_1})}\inv  \  , \  \dots \ , \  Q_{s(e_{k_r})}\inv\right)   \\
	& = A_{\to i} \circ \text{\rm Diag}\left(Q_{s(e_{k_1})}  \  \dots \ , \  Q_{s(e_{k_r})}\right) \circ \text{\rm Diag}\left(Q_{s(e_{k_1})}\inv  \  , \  \dots \ , \  Q_{s(e_{k_r})}\inv\right)  = A_{\to i}
	\end{align*}
	On the other hand, if $i$ is a sink, then $R_{\to i} = A_{\to i}^\prime$ and $Q_i = \id_{d_i}$. As similar calculation as above shows that $(\bfQ \cdot \bfR)_{\to i} =  A_{\to i}$. 	We note that the representation $\bfA$ is determined by the matrices $A_{\to i}$ as $i$ ranges over the non-source vertices. This finishes the proof. 
\end{proof}

\subsection{Relation to neural networks}
We now explain the relation between Algorithm \ref{alg:QR-rescaling} and Algorithm \ref{alg:QR-decomp-quiver}. 
First, we note that a more general version of Proposition \ref{prop:qr-quiver} holds, namely, each representation admits a decomposition $\bfW = \bfQ \cdot \bfR$,  where each  $Q_i$ is orthogonal, and each $R_{\to i}$ becomes upper triangular   after an appropriate permutation of the columns. Indeed, the only necessary modifications to Algorithm \ref{alg:QR-decomp-quiver} are that one must:
\begin{enumerate}
		\setlength\itemsep{5pt}
	\item fix a permutation matrix $P_i \in \R^{d_{\to i } \times d_{\to i}}$ for each non-source vertex $i$,
	\item perform $W_i \gets W_i \circ P_i\inv$ before computing the QR decomposition $W_i = Q_i \circ R^\prime_i$, and 
	\item set $R_{\to i} = R_{\to i}^\prime \circ P_i$. 	 
\end{enumerate}

We now consider, for each non-sink vertex $i$, the  permutation   of the standard basis of $\R^{d_{\to i}} = \R^{d_{j_1}} \oplus \cdots \oplus \R^{d_{j_r}}$ in which all the standard basis vectors of  $\R^{\dred_{\to i}}  =R^{\dred_{j_1}} \oplus \cdots \oplus \R^{d_{j_r}}$ appear first, in order, before the remaining basis vectors\footnote{To be more explicit, let $\epsilon^{(1)}_1, \dots, \epsilon^{(1)}_{\dred_{j_1}}$ be the standard basis of $\R^{d_{j_1}}$, let $\epsilon^{(2)}_1, \dots, \epsilon^{(2)}_{d_{j_2}}$ be the standard basis of $\R^{d_{j_2}}$, and so forth.  These combine, in order, to produce the standard basis of $\R^{d_{\to i}} = \R^{d_{j_1}} \oplus \cdots \oplus \R^{d_{j_r}}$. 
Consider the new basis of $\R^{d_{\to i}}$ starting with
$\epsilon^{(1)}_1, \dots, \epsilon^{(1)}_{\dred_{j_1}}, \epsilon^{(2)}_1, \dots, \epsilon^{(2)}_{\dred_{j_2}}, \dots, \epsilon^{(r)}_1, \dots, \epsilon^{(r)}_{\dred_{j_r}}$
and ending with the remaining standard basis vectors (in any order).}.
Hence we have a permutation matrix $P_i \in \R^{d_{\to i } \times d_{\to i}}$ for each non-sink vertex $i$. We now apply the modification of Algorithm \ref{alg:QR-decomp-quiver} described above to produce $\bfQ$ and $\bfR$. The matrices $\bfQ$ will be the same as those in Algorithm \ref{alg:QR-rescaling}. For each edge $e$, take $W\red_e$ to be the top left $d_{t(e)} \times d_{s(e)}$ block of each $R_e$.  These are the matrices $\bfW\red_e$ appearing in Algorithm \ref{alg:QR-rescaling}. 

Let $\quiver$ be a neural quiver and let $\bfd$ be a dimension vector for $\quiver$.  Recall from Section \ref{sec:proj-gd} that for any activation functions $\boldrho$  and any batch of training data we have a loss function $\L :\Parqd \to \R$ and gradient descent map
$\gamma : \Parqd \to \Parqd$. (These maps can be equivalently regarded on $\Rep(\quiver, \bfd)$.)
The following result shows that gradient descent optimization starting at $\bfW$ is equivalent to gradient descent optimization starting at $\bfR$. 

\begin{cor}
	Let $\bfW$ be a representation of $\quiver$ and  let $\bfW = \bfQ \cdot \bfR$ be its QR decomposition. Fix radial rescaling activations $\boldrho$. Then, for any $k \geq 0$, we have: 	$$\gamma^k (\bfW) = \bfQ \cdot \gamma^k(\bfR)$$
\end{cor}

\begin{proof}
	[Sketch of proof.] Note that the action of $\bfQ$ on $\Rep(\quiver, \bfd)$ is an orthogonal transformation. Moreover, this action does not change the feedforward function, due to the fact that orthogonal transformations commute with radial functions (see Lemma  \ref{lem:rad-resc-orth}). Consequently, the action of $\bfQ$ commutes with any number of steps of gradient descent  (see Lemma 21  of   \cite{ganev_universal_2022}).  
\end{proof}

\vfill
\end{document}